\documentclass[11pt]{article}

\usepackage{standard}

\newcommand{\mlp}{\mathsf{mlp}}
\newcommand{\attn}{\mathsf{attn}}
\newcommand{\tran}{\mathsf{tran}}
\newcommand{\SeqComp}{\mathsf{FuncComp}}

\usepackage{bbm}

\usepackage[breakable]{tcolorbox}
\newtcolorbox{construction}[2][]
{
    breakable,
    colframe = gray!50,
    colback  = gray!10,
    coltitle = gray!10!black,
    before skip = 10pt,
    after skip = 10pt,
    title    = \textbf{#2},
    #1,
}

\title{Theoretical limitations of multi-layer Transformer}
\author{
Lijie Chen \\ UC Berkeley \\ \texttt{lijiechen@berkeley.edu}
\and 
Binghui Peng\\ Simons Institute, UC Berkeley  \\ \texttt{bp2601@columbia.edu}
\and 
Hongxun Wu \\ UC Berkeley \\ \texttt{wuhx@berkeley.edu}
}
\date{}

\begin{document}
\maketitle

\begin{abstract}
Transformers, especially the decoder-only variants, are the backbone of most modern large language models; yet we do not have much understanding of their expressive power except for the simple $1$-layer case. 

Due to the difficulty of analyzing multi-layer models, all previous work relies on unproven complexity conjectures to show limitations for multi-layer Transformers. In this work, we prove the first \emph{unconditional} lower bound against multi-layer decoder-only transformers. For any constant $L$,
we prove that any $L$-layer decoder-only transformer needs a polynomial model dimension ($n^{\Omega(1)}$) to perform sequential composition of $L$ functions over an input of $n$ tokens.

As a consequence, our results give: (1) the first depth-width trade-off for multi-layer transformers, exhibiting that the $L$-step composition task is exponentially harder for $L$-layer models compared to $(L+1)$-layer ones; (2) an unconditional separation between encoder and decoder, exhibiting a hard task for decoders that can be solved by an exponentially shallower and smaller encoder; (3) a provable advantage of chain-of-thought, exhibiting a task that becomes exponentially easier with chain-of-thought. 

On the technical side, we propose the multi-party {\em autoregressive communication model} that captures the computation of a decoder-only Transformer. We also introduce a new proof technique that finds a certain \emph{indistinguishable decomposition} of all possible inputs iteratively for proving lower bounds in this model. We believe our new communication model and proof technique will be helpful to further understand the computational power of transformers.
\end{abstract}

\clearpage
\newpage

\section{Introduction}
\label{sec:intro}

%Transformer \cite{vaswani2017attention} is the backbone architecture of modern large language models. When pre-training over huge amount of corpus and fine-tuning over expert datasets, Transformers achieve remarkable performance across diverse natural language tasks \cite{achiam2023gpt} and exhibit emergence of intelligence \cite{bubeck2023sparks}.

%There is no doubt that Transformer is a very ingenious and powerful architecture, as evidenced by the huge empirical success. Nevertheless, is there any limitation/weakness of the architecture? We believe this is a fundamental question given the wide-deployment LLMs.In this paper, we study the question from the computation (or representation) perspective, which views neural net as parameterized functions and asks what they can efficiently compute.

The Transformer architecture \cite{vaswani2017attention} forms the backbone of modern large language models (LLMs). When pre-trained on vast corpora and fine-tuned on expert datasets, Transformers achieve impressive performance across a range of natural language tasks \cite{achiam2023gpt} and demonstrate emergent intelligence \cite{bubeck2023sparks}. 

There is no doubt that the Transformer is an ingenious and powerful architecture, as evidenced by many of its substantial empirical successes. Nevertheless, {\em does the architecture have any potential limitations or weaknesses?} We believe this is a fundamental question, especially given the widespread deployment of LLMs. In this paper, we investigate this question from a computational (representational) perspective, viewing neural networks as parameterized functions and studying what they can compute efficiently.

%There is a long list of papers on the representation power of Transformer \cite{hahn2020theoretical}.
%Prior to work, we have good understanding on one-layer Transformer -- Unconditional lower bounds exist for various simple tasks, including composition of two functions \cite{peng2024limitations}, 3SUM \cite{sanford2023representational} and induction heads \cite{sanford2024one}.
%However, for the arguably more interesting regime of {\em multi-layer} Transformer, we do not know any unconditional lower bound (even for two-layer Transformer). 
%Indeed, there is a speculation in the literature that multi-layer Transformer has circuit lower bound barriers, i.e., one needs first prove lower bounds against threshold circuits (see \cite{hahn2020theoretical} for a detailed discussion).
%Therefore, the literature turns to characterize the limitation of Transformer by relying on certain computational conjecture, e.g., the computation power of Transformer can be captured by constant depth threshold circuit $\mathsf{TC}^0$ \cite{merrill2023parallelism}, log-space computation $\mathsf{L}$ \cite{peng2024limitations} and Massively Parallel Computation (MPC) \cite{sanford2024transformers}.

The literature on the representation power of Transformers is extensive (see Section \ref{sec:related} for a detailed discussion). Prior to our work, we only have a solid understanding of one-layer Transformers -- unconditional lower bounds have been established for various basic tasks, such as the composition of two functions \cite{peng2024limitations}, 3SUM \cite{sanford2023representational}, and induction heads \cite{sanford2024one}. However, for the arguably more interesting case of multi-layer Transformers, unconditional lower bounds remain elusive even for two-layer models. There is even speculation in the literature that multi-layer Transformers face circuit lower bound barriers, meaning that proving lower bounds for multi-layer Transformer may require first resolving long-standing open qestion on circuit lower bounds, see \cite{hahn2020theoretical} for a detailed discussion. As a result, researchers have turned to characterizing the limitations of Transformers through computational conjectures, proving that the computational power of Transformers is contained within constant depth threshold circuit $\mathsf{TC}^0$ \cite{merrill2023parallelism}, log-space computation $\mathsf{L}$ \cite{peng2024limitations} and Massively Parallel Computation (MPC) \cite{sanford2024transformers}.

%\lijie{I feel one point is that it seems these work all study encoders instead of decoders... decoder is not subject to circuit lower bound barrier I guess?}

\paragraph{``Small'' vs. ``Large'' transformers.} To study the transformer as a computational model, we often consider the context length (prompt length) $n$ as a growing parameter and study the required size of the transformer (in terms of parameters such as head embedding dimension $d$ and number of attention heads $H$) for solving particular problems. Note that the total number of parameters of a transformer is roughly $L \cdot \poly(d H p)$, which is independent of the context length $n$, where $L$ is the number of layer, $p$ is the number of bit precision for each entry in the embedding and one can often think $p$ as $\Theta(\log n)$. 
 Following the convention of~\cite{sanford2023representational}, we consider the transformer small if $dHp \le n^{o(1)}$, and large if $dHp \ge n^{\eps}$ for some constant $\eps > 0$. 

\subsection{Our result}

In this paper, we prove the first {\em unconditional} lower bounds for any constant layer (decoder-only) Transformer. Indeed, we prove that no small Transformer can not solve sequential composition tasks over long context; see~\Cref{sec:pre:transformer} for a formal definition of transformers.

\begin{theorem}[Lower bound for multi-layer Transformer]
\label{thm:main}
Let $H$ be the number attention heads, $d$ be the head dimension, $p$ be the precision, $L$ be the number of layers, $n$ be the prompt length. For any $L \leq \wt{O}(\log\log(n))$, an $L$-layer decoder-only Transformer could not solve $L$-sequential function composition whenever 
$Hdp \leq n^{2^{-4L}}$.
%the prompt length $n \geq (Hdp)^{2^{4L}}$.
\end{theorem}

A formal definition of the $L$-sequential function composition can be found at Definition \ref{def:composition}. Roughly speaking, given $L$ functions $f_1, \ldots, f_{L}$ and a query $w = (w_1, \ldots, w_{L})$, it asks to compute $i_1 = f_1(w_1), i_2 = f_2(w_2, i_1), \ldots, i_{L} = f_L(w_{L}, i_{L-1})$.
%\hongxun{Comparing with 1.3.2, the index of $i$ and $w$ is a bit inconsistent. Is it intentional?}\Binghui{How about we replace all $w_1,\ldots, $ with $w$?}

%We first leave some remarks on our main Theorem \ref{thm:main}
%\begin{remark*}[Encoder vs. decoder] 
\paragraph{Encoder vs. decoder.}
The focus of our paper is on decoder-only Transformer, which is the most popular architecture among all LLMs. The original Transformer paper \cite{vaswani2017attention} consists of two main components: the {\em encoder} and the {\em decoder}. The encoder processes input tokens with pairwise attention mechanisms, capturing contextual relationships across the input. In contrast, the decoder employs causal masking to attend only to previous tokens, enabling autoregressive sequence generation. 

Indeed, the history of LLMs reflects a shift from encoder-based to decoder-only architectures. Early models like BERT \cite{devlin2018bert}, built on the Transformer’s encoder, excelled in understanding tasks by learning contextual relationships through masked language modeling. However, encoder-based models have been proved limited for text generation, as they lacked autoregressive capabilities. Decoder-only models are trained through next-token prediction, which not only aligns well with text generation tasks but also improves computational efficiency during generation. Following the success of GPTs \cite{radford2019language,mann2020language}, all prominent large language models, including Claude \cite{claude2024introducing}, Gemini \cite{team2024gemini} and LLaMA \cite{meta2024introducing}, have adopted the {\em decoder-only} approach. 
%\end{remark*}

%\begin{remark*}[Composition]
\paragraph{Composition.}
The ability to perform compositional tasks has been a central focus of empirical research \cite{wei2022emergent, press2023measuring, dziri2023faith, arora2023theory,yang2024large, wang2024alpine, petty2024impact, jelassi2024mixture, ye2024physics}, as compositionality is essential for reasoning and handling complex tasks. 
\cite{dziri2023faith} demonstrated through extensive experiments that Transformers struggle with tasks requiring the sequential composition of elementary steps, such as multiplying multi-digit integers and solving logical puzzles, with performance rapidly declining as the depth of composition increases. 
Our main theorem, Theorem \ref{thm:main}, provides theoretical justification for the limitations of Transformers in executing sequential composition.

On the other hand, Theorem \ref{thm:main} highlights the importance of depth in Transformers. As we elaborate in Section \ref{sec:application}, Theorem \ref{thm:main} implies a depth-size tradeoff for Transformers in compositional tasks. This aligns with the empirical findings of \cite{ye2024physics}, which demonstrate that depth plays a more critical role than width in reasoning and composition tasks.
%It is conjectured in \cite{wei2022emergent} that the emergent phenomenon LLMs is (partially) credit to the increasing reasoning capability of larger language model, when 

%\end{remark*}
%\Binghui{Elaborate more this point}
%\cite{ye2024physics,petty2024impact,wei2022emergent}

\paragraph{On the circuit lower bound barrier.} It was suggested in~\cite{hahn2020theoretical} that unconditional lower bounds against \emph{encoder-only} transformer would imply breakthrough circuit lower bounds against linear threshold circuits.\footnote{In~\cite{hahn2020theoretical} it was mentioned briefly in a footnote without providing a detailed argument; in~\cref{app:encoders-lowb-imply-ckt-lowb}, we provide a detailed argument showing that an $n^{\eps}$-size lower bounds against encoders would imply breakthrough circuit lower bounds.} Our result avoids this barrier by exploiting the information bottleneck and autoregressive nature of \emph{decoder-only} models. Namely, our proof crucially depends on the fact that in decoder-only models, each token cannot attend to any token after it.

%\lijie{I feel like we should try to formalize this a bit, what~\cite{hahn2020theoretical} wrote does not make sense to me... After the deadline, we probably should flesh this out a bit.}

\subsection{Applications}
\label{sec:application}
%Theorem \ref{thm:main} has several important implications.

\subsubsection*{Application 1: Depth-size tradeoff for Transformer}

Depth-size (or depth-width) tradeoff has been extensively studied for neural networks \cite{martens2013representational,eldan2016power, telgarsky2016benefits,liang2016deep, daniely2017depth, safran2017depth,lu2017expressive, yarotsky2017error,safran2019depth,vardi2020neural, bresler2020sharp, vardi2021size, levine2020limits}. 
Most work has focused on the feed-forward ReLU network.
However, as pointed out by \cite{vardi2020neural, vardi2021size}, proving depth-width tradeoff for $L \geq 4$ layer ReLU network faces circuit lower bound and natural proof barriers for benign functions (i.e., functions that can be computed in polynomial time and has polynomial-bounded Lipschitz constant). 
Hence, existing work \cite{eldan2016power, daniely2017depth} either focus on very shallow network ($L=2,3$), or construct separations based on high oscillatory function with super-polynomial Lipscitz constant \cite{telgarsky2016benefits}.

In sharp contrast, Theorem \ref{thm:main} implies an exponential depth-size tradeoff for Transformer architecture. This is because the $L$-sequential function composition can be easily solved by an $(L+1)$-layer decoder-only Transformer, with only polylogarithmic number of parameters (i.e., $H = O(1),d = \poly\log(n), p = \log(n)$).
\begin{corollary}[Depth-size tradeoff]
\label{cor:depth-size}
For any constant $L \geq 1$, there exists a task (a.k.a. $L$-sequential function composition) such that (1) an $(L+1)$-layer Transformer could solve the task with polylogarithmic number of parameters while (2) any $L$-layer Transformer needs polynomial number of parameters to solve the task.
\end{corollary}

\subsubsection*{Application 2: Separation between Transformer encoder/decoder}

Theorem \ref{thm:main} also implies a separation between the encoder and the decoder architecture, since the $L$-sequential function composition task can be easily solved by an $O(\log(L))$-layer Transformer encoder.

\begin{corollary}[Separation between encoder and decoder]
\label{cor:encoder}
For any constant $L \geq 1$, there exists a task (a.k.a. $L$-sequential function composition) such that (1) an $O(\log(L))$-layer Transformer encoder could solve the task with polylogarithmic number of parameters while (2) any $L$-layer Transformer decoder needs polynomial number of parameters to solve the task.
\end{corollary}

%\Binghui{Any literature study the distinction between encoder and decoder?}

There has been a lot of work~\cite{fu2023decoder, allen2023physics1, allen2023physics31, nielsen2024encoder, qorib2024decoder} comparing the empirical performance of encoder and decoder architecture, see \cite{tay24encoder} for a detailed coverage. A recent work~\cite{ewer2024entp} compares the expressive power between encoder and decoder, showing a strong separation by assuming a conjecture of the hardness of a certain triplet counting problem. In contrast, our work gives the first unconditional separation without any unproven assumptions. 

%\lijie{There are some empirical works apparently, see Section 2.1 of~\url{https://arxiv.org/pdf/2406.13469v1}}

\subsubsection*{Application 3: Provable benefits of chain of thought}
The chain of thought (CoT) \cite{wei2022chain} is known to help with the reasoning by inducing the LLM to perform step by step reasoning and eventually leading to the correct answer.
From a theoretical view, CoT provides Transformer with extra computation space, and previous work \cite{perez2021attention, feng2023towards, merrill2023expresssive, li2024chain} proved that log-precision Transformer with CoT could simulate any polynomial-time algorithm.
Therefore, by further assuming certain complexity conjecture (e.g. $\mathsf{P} \not\subset \mathsf{TC}^{0}$), their results imply that constant depth Transformer with CoT could simulate poly-time algorithm, while constant depth Transform ($\subseteq \mathsf{TC}^{0}$) itself can not solve $\mathsf{P}$-complete task.

Theorem \ref{thm:main} implies the first provable benefits of CoT, without relying on any computational complexity conjecture. This is because $L$-sequential function composition can be easily solved using $L$-steps of CoT with only polylogarithmic number of parameters.

\begin{corollary}[Provably benefits of CoT]
\label{cor:cot}
For any constant $L \geq 1$, there exists a task (a.k.a. $L$-sequential function composition) such that (1) an one-layer Transformer with CoT could solve the task with polylogarithmic number of parameters while (2) any $L$-layer Transformer decoder needs polynomial number of parameters to solve the task.
\end{corollary}

%Indeed, a simple CoT scheme can plausibly mitigate our impossibility result on composition by generating a short prompt. However, we also prove a theorem implying that a Transformer layer with CoT needs far 

\subsection{Technique overview}

Below, we provide an overview of~\Cref{thm:main}. In~\Cref{sec:tech-overview:cc}, we introduce the autoregressive communication model, which captures the computation of decoder-only transformers. Then, in~\Cref{sec:tech-overview:seq-func}, we define the $L$-sequential function composition task, for which we will show hardness against decoder-only models. Finally, in~\Cref{sec:tech-overview:low-b}, we sketch the proof that $L$-sequential function composition is hard in the autoregressive communication model, which implies~\Cref{thm:main}.

\paragraph{Notation.} We write $[n] = \{1, 2, \ldots, n\}$ and $[n_1:n_2] = \{n_1, n_1+1, \ldots, n_2\}$. In the corner case, we let $[0] = \emptyset$.

\subsubsection{Autoregressive communication model}\label{sec:tech-overview:cc}
In order to capture the computational power of autoregressive models such as decoder-only Transformers, we introduce the autoregressive communication model, which is the key conceptual contribution of this paper.

\paragraph{Setting.} A protocol in the autoregressive communication model proceeds in $L$ epochs (which is also the number of layers in a corresponding transformer). There are $N$ players, each player $i \in [N]$ receives $z_{i}$ as input. Their goal is to let player $1$ compute an intended function $f(z_1,\dotsc,z_{N})$ at the end of $L$-th epoch.

It is also helpful to imagine players are arranged on a line as player $N,N-1,\dotsc,1$; so autoregressively, player $i$ can only attend to (i.e., send message to) player $j$ such that $j > i$.

\paragraph{Communication.} For $\ell \in [0:L]$, let $X_{i}^{(\ell)}$ be the message collected by player $i$ ($i \in [N]$) after the $\ell$-th epoch of communication. Initially, when $\ell = 0$, $X_{i}^{(0)}$ is just the input of player $i$. 

For $\ell = 1,2,\ldots, L$, the $\ell$-th epoch of communication proceeds as follows. For player $i \in [N]$:
\begin{itemize}
\item The player $i$ sends a message $\Gamma_{i, j}^{(\ell)}$ to all players $j \in [i+1: N]$. 
\item The player $j \in [i+1: N]$, based on its own information $X_{j}^{(\ell-1)}$ and player $i$'s message $\Gamma_{i, j}^{(\ell)}$, it sends a message $\Pi_{j, i}^{(\ell)}$ to player $i$. 
%The length of the message satisfies 
%\[
%|\Pi_{j, i}^{(\ell)}| = 2 B \cdot m_{(i)}.
%\]
\item Finally, the player $i$ updates its collection of information as
\[
X_{i}^{(\ell)}:= X_{i}^{(\ell-1)} ~\cup~ \bigcup_{j > i} \Pi_{j, i}^{(\ell)}.
\]
That is, the information state of player $i$ is updated to include all messages received from players $j \in [i+1:N]$.
\end{itemize}
Finally, the player $1$ returns an output based on its information state $X_{1}^{(L)}$.

{\em The most salient feature of the autoregressive communication model is that the players are forgetful.} That is, the player $j$ does not remember anything sent from player $i \in [1:j-1]$; see~\Cref{sec:reduction} for a formal definition of the autoregressive communication model.

\paragraph{Transformer as autoregressive communication.} The Transformer architecture can be captured as by the autoregressive communication model. In particular, if we partition the input prompt as $(z_{N}, \ldots, z_1)$ (where each $z_i$ can contain multiple tokens), the Transformer can be seen as a special autoregressive communication protocol, where each attention layer implements one epoch of communication and the MLP layers between the attention layers are used to perform local computation. The message $\Gamma_{i, j}^{(\ell)}$ contains the queries from positions corresponding to tokens of $z_i$ and the returned message $\Pi_{j, i}^{(\ell)}$ contains the partial sum of key/value. Moreover the size of $\Gamma_{i, j}^{(\ell)}$ and $\Pi_{j, i}^{(\ell)}$ is bounded by $2Hdp \cdot |z_i|$, that is, they are proportional to the input length of player $i$. We call $B := 2Hdp$ the message bits of the communication model; see~\Cref{lem:reduction} for a detailed proof of the simulation of decoder-only transformers by autoregressive communication protocols.

\subsubsection{Sequential function composition}\label{sec:tech-overview:seq-func}
Now we elaborate on the $L$-sequential function composition task and explain why an autoregressive communication protocol with small message bits $B = 2Hdp$ fails to solve it.

\paragraph{Intuition.} We first provide intuitions on what makes a task hard for autoregressive communication models. Intuitively, the player $i$ has stronger communication power than player $j > i$, since player $i$ could communicate to all players $[i+1: N]$ and it remembers all their returned message in the information state. 
Therefore, the failure of an autoregressive communication protocol happens in the regime of $|z_N| \gg |z_{N-1}| \gg \cdots \gg |z_{1}|$.\footnote{Note that on the other extreme, when $|z_{1}| \geq |z_j|$ ($\forall j \in [2:N]$), there exists a trivial communication protocol since the player $j$ could sends it input $z_j$ to player $1$.} On the other hand, in this regime, player $j$ possess much more information than player $i$ regard the entire sequence $(z_{N}, \ldots, z_{1})$ since its input has larger size and its communication capacity is larger (i.e., $|z_{j}|\cdot 2Hdp \gg |z_{i}|\cdot 2Hdp$). In order to avoid shortcut in the communication, we must make sure that player $1$ holds important ``secrets'' that are crucial for {\em all} players $j \in [2:N]$. 

To this end, consider the $L$-sequential function composition task. Let $m, n_1, \ldots, n_{L-1}$ be parameters and $N_{\ell} = m\cdot \prod_{i=1}^{ \ell}n_{i}$ for $\ell \in [0 : L]$.

\paragraph{$L$-sequential function composition.} Our task, $L\text{-}\SeqComp(w,z_0,z_1,\dotsc,z_{L})$, takes a sequence of functions $z_{0}, z_1 \ldots, z_{L}$ as input, where $z_0 \in [m]$ and $z_{\ell}: [N_{\ell-1} ] \rightarrow [N_{\ell-1}]$ for $\ell \in [L]$ and a query $w = (w_{1}, \ldots, w_{L-1}) \in [n_{1}] \times \cdots \times [n_{L-1}]$. The output is defined inductively as follows: First, one computes
\begin{align*}
i_{0} = z_0 \in [m], \quad i_1 = z_1(i_0) \in [N_0]
\end{align*}
and one inductively computes, for each $\ell = 1,2,\ldots, L-1$:
\begin{align*}
i_{2} = z_2(w_1, i_1) \in [N_1], 
%\quad i_3 = z_3(w_2, i_2) \in [N_2],  
\quad \ldots , \quad i_{\ell+1} = z_{\ell+1}(w_{\ell}, i_{\ell}) \in [N_{\ell}] .
\end{align*}
The final output is then defined as 
\[
L\text{-}\SeqComp(w,z_0,z_1,\dotsc,z_{L}) = i_{L}.
\]

In the autoregressive communication model, for the $L$-sequential function composition task, we have $N = L+2$ parties. For the sake of exposition, we rename them as player $L, L - 1, \dots, 0, -1$, where the player $\ell$ receives $z_\ell$ ($\ell\in [0:L]$) and the player $-1$ receives the query $z_{-1}:= w$. 
%We also use player $\e$ to denote the player $-1$ to avoid confusion.
For simplicity, we can assume that the message $\Gamma_{i, j}^{(\ell)}$ is just the whole information state $X_{i}^{(\ell-1)}$ of player $i$ at the end of epoch $\ell-1$.

\subsubsection{Communication lower bound}\label{sec:tech-overview:low-b}

Next, we elaborate on the communication lower bound for the sequential function composition task, which is the main technical part of this paper.  
In this overview, we hide the precise choice of parameters and highlight the key ideas behind the proof.

For convenience, we set $A_{\ell} = [N_{\ell-1}]^{N_{\ell-1}}$ be the input domain of player $\ell$ ($\ell \in [1:L]$). We also set $A_0 = [m]$ and $A_{-1} = [n_1] \times \cdots \times [n_{L-1}]$ be the input domain of player $0$ and player $-1$, respectively.

%\paragraph{Warm up: Lower bound for $L = 2$}
%We start with the simplest non-trivial case of $L=2$. That is, there are two epochs of communication between four players. The players $-1$ receives $w = w_1 \in [n_1]$, the player $0$ receives $z_0 \in [m]$, the player $1$ receives $z_1 \in A_1:= [N_0]^{N_0} = [m]^{m}$ and the player $2$ receives $z_2 \in A_2: = [N_1]^{N_1} = [mn_1]^{mn_1}$.

%At a high level, we want to find a set $R_2 \subseteq A_2$ and sets $Z_1 \subseteq A_1$ $Z_{0} \subseteq [m]$, such that, for every possible inputs of 

%\paragraph{Lower bound for constant $L$}
\paragraph{Indistinguishable decomposition.} Our key idea is the concept of \emph{indistinguishable decomposition}, which is two sets $R_{\ge \ell}$ and $Z_{<\ell}$, where $R_{\ge \ell}$ is a set of input assignments to players $[\ell:L]$ and $Z_{<\ell}$ is a set of input assignments to players $[-1:\ell-1]$, such that for every possible inputs $z_{<\ell} \in Z_{<\ell}$, all assignments in $R_{\ge \ell}$ are indistinguishable to players $[-1:\ell-1]$ on inputs $z_{< \ell}$ after $\ell$ epochs (because they lead to the same transcripts).

For technical reasons, we will further assume (1) $Z_{<\ell}$ is a product set $Z_{<\ell} := Z_{-1} \times Z_0 \times \cdots \times Z_{\ell - 1}$, where $Z_{i} \subseteq A_i$ for every $i \in [-1 : \ell - 1]$; and (2) $Z_{-1} = A_{-1}$.

Indistinguishable decomposition is helpful because when $\ell = L$, for every input assignment from $Z_{<L}$ to players $[-1:L-1]$, the player $-1$ after $L$ epochs (i.e., at the end of the protocol) sees the same transcript when player $L$ receives different inputs from $R_{\ge L}$.\footnote{Due to our renaming, player $-1$ is the player $1$ in the definition of the autoregressive model, who is supposed to output an answer at the end of the protocol.} In particular, it means for every $\wt{z}_{<L} \in Z_{< L}$, the answer $L$-$\SeqComp(\wt{z}_{<L}, \wt{z}_L)$ must be the same for every $\wt{z}_L \in R_{\ge L}$. This is a strong constraint on the set $R_{\ge L}$ that we will analyze below.

First, we define $\mathcal{I}_{\ell-1} := \mathcal{I}_{\ell-1}(Z_{<L})$ as the set of all partial composition values $i_{\ell-1}$ for inputs from $Z_{<\ell}$.

Recall that by its definition, we have
\[
L\text{-}\SeqComp(\wt{z}_{<L}, \wt{z}_L) = i_{L} = \wt{z}_{L}(\wt{w}_{L-1}, i_{L-1}).
\]

Fixing any $\wt{z}_{<L} \in Z_{< L}$, this value is the same for all $\tilde{z}_L \in R_{\geq L}$. As $Z_{<L}$ is the product set $A_{-1} \times Z_0 \times \cdots \times Z_{L-1}$ and $\wt{z}_{<L} \in Z_{< L}$, variable $\wt{w} = \wt{z}_{-1}$ can take any value $A_{-1}$. So $\wt{w}_{L-1}$ can take any value in $[n_{L-1}]$. Fixing any $\wt{z}_{<L} \in Z_{< L}$ is equivalent to fixing any $\wt{w}_{L-1} \in [n_{L-1}]$ and $i_{L - 1} \in \mathcal{I}_{L-1}$, and the value of $\wt{z}_{L}(\wt{w}_{L-1},i_{L-1})$ is fixed for every $\wt{z}_{L} \in R_{\ge L}$. This in particular, means that
\begin{equation}
    |R_{\ge L}| \le \frac{|A_{L}|}{ N_{L-1}^{|\mathcal{I}_{L-1}| \cdot n_{L-1}} }. \label{eq:bound-on-RgeL}
\end{equation}

%we define $\mathcal{I}_{\ell-1}$ as
%\begin{align}
%\mI_{\ell-1}(Z_{<\ell}):= \{\wt{i}_{\ell-1}: \wt{i}_{\ell-1} = i_{\ell-1}(\wt{z}_{-1}, \wt{z}_0, \ldots, \wt{z}_{\ell-1}) \text{ for some } (\wt{z}_{-1},\dotsc, \wt{z}_{\ell-1}) \in Z_{<\ell} \}. \notag
%\end{align}
%That is, $\mI_{\ell-1}(Z_{<\ell})$ is the set of all partial composition values for inputs from $Z_{<\ell}$.

Given the above~\eqref{eq:bound-on-RgeL}, our goal now is (for every autoregressive communication protocol $\Pi$ with $L$ epochs and $B$ message bits) to find an indistinguishable decomposition with the following two properties:
\begin{itemize}
    \item (\textbf{Large remaining entropy}) The size of $R_{\ge L}$ is large; and
    \item (\textbf{Large cover}) The size of $\mathcal{I}_{L-1}(Z_{<L})$ is large.
\end{itemize}

We intentionally omit the precise meaning of ``large'' in the two items above; the idea is that these two conditions together would contradict~\eqref{eq:bound-on-RgeL}, therefore showing $\Pi$ does not compute $L$-$\SeqComp$. Since $\Pi$ is an arbitrary protocol, this proves the lower bound.

\paragraph{Construct indistinguishable decomposition via induction.} We construct the indistinguishable decomposition $R_{\geq \ell}$ and $Z_{< \ell}$ via an induction on $\ell$. Suppose we have already obtained $R_{\geq \ell}$ and $Z_{< \ell}$, and let $\Lambda^{(\ell)} = \left\{ \Lambda^{(\ell)}_{\wt{z}_{<\ell}} \right\}_{\wt{z}_{<\ell} \in Z_{<\ell}}$ be the collection of transcripts from players $[\ell:L]$ to players $[-1:\ell]$, when players $[\ell:L]$ take input assignments from $R_{\geq\ell}$ and players $[-1: \ell-1]$ take input assignments from $Z_{< \ell}$. Indistinguishable decomposition ensures that $\Lambda^{(\ell)}$ is well-defined (e.g., it is the correct transcript for every $z_{\geq \ell} \in R_{\geq \ell}$).

To construct $R_{\geq \ell+1}$ and $Z_{<\ell+1}$, we proceeds in two parts.

\begin{itemize}
\item ({\bf Part 1}) First, we wish to select a set $Z_{\ell} \subseteq A_{\ell}$ and let $Z_{< \ell+1}= Z_{< \ell} \times Z_{\ell}$. To this end, we find a rectangular subset from $R_{\ge\ell}$, i.e., we find $S_{\ge \ell+1} \times Z_{\ell} \subseteq R_{\ge \ell}$ where $S_{\ge \ell+1} \subseteq A_L \times \cdots \times A_{\ell+1}$ and $Z_{\ell} \subseteq A_\ell$, with the requirement that (1) $S_{\ge \ell+1}$ has large size, and (2) the size of $\mI(Z_{<\ell+1})$ is large.
\item ({\bf Part 2}) Next, we want to distill a subset $R_{\ge \ell+1}$ from $S_{\ge \ell+1}$. This requires us to determine the transcripts from players $[\ell+1: L]$ to players $[-1:\ell]$ in the first $\ell+1$ epochs, when they receive input assignments from $R_{\ge \ell+1}$ and $Z_{< \ell+1}$. We further divides into three steps:
\begin{enumerate}
\item ({\bf Step 1}) We first determine transcripts to players $[-1:\ell-1]$ over the first $\ell$ epochs. We can simply use $\Lambda^{(\ell)}$ (i.e., transcripts from the inductive hypothesis). Up to this point, all assignments to $S_{\ge \ell+1}$ are indistinguishable to $Z_{< \ell+1}$.
\item ({\bf Step 2}) Next, we determine transcripts to players $[-1:\ell-1]$ in the $(\ell+1)$-th epoch. This is the key step of the whole proof. Our key insight is that $Z_{\ell}$ is indistinguishable to players $[-1:\ell-1]$ after $\ell$ epochs, when they take input from $Z_{< \ell}$. Hence, the $(\ell+1)$-th epoch transcripts to players $[-1:\ell-1]$ are independent of the choice of $z_{\ell}\in Z_{\ell}$. We can then use a greedy strategy to select transcripts that leak the least amount of information about $S_{\ge \ell+1}$, without consulting to the value of $z_{\ell} \in Z_{\ell}$.\footnote{This saving is crucial for our proof to work.} After this step, we are left with a large subset $T_{\ge \ell+1} \subseteq S_{\ge \ell+1}$ that is indistinguishable to $Z_{< \ell+1}$.
\item ({\bf Step 3}) Finally, we determine transcripts to players $\ell$ over the first $\ell+1$ epochs, we simply use a greedy strategy to select transcripts that leak the least amount of information about $T_{\ge \ell+1}$, and we finally left with a large set $R_{\ge \ell+1}\subseteq T_{\geq \ell+1}$ that is indistinguishable to $Z_{< \ell+1}$.
\end{enumerate}
\end{itemize}

%This is a strong requirement and we will carefully define properties of $R_{\ge \ell}$ and $Z_{< \ell}$ such that this requirement would lead to contradiction, and therefore obtaining our lower bound.

\subsection{Related work}
\label{sec:related}

\paragraph{Representation power of Transformer.}  There is a long line of work studying the representation power of Transformer \cite{ebrahimi2020can, yao2021self, perez2021attention,chiang2022overcoming, sanford2023representational, liutransformers2023, wen2023transformers, wen2024rnns,jelassirepeat2024,  feng2023towards, merrill2023expresssive, li2024chain}. Starting from the work of \cite{yao2021self,hewitt2020rnns}, a line of work demonstrates the advantage of Transformer against recurrent architecture (RNNs, LSTMs and state space model) on a variety of tasks, such that parsing hierarchical structure \cite{yao2021self}, sparse averaging \cite{sanford2023representational}, in context learning \cite{wen2024rnns} and copying \cite{jelassirepeat2024}. 

Hahn \cite{hahn2020theoretical} initiated the study on the limitation of the self-attention mechanism; the author proves that a Transformer with hard attention can not recognize parity and dyck language. The lower bound does not extend to soft-attention since both parity and dyck language can be solved by two-layer Transformers.
Since then, the literature either uses communication complexity to prove lower bound for one-layer Transformer \cite{peng2024limitations, sanford2023representational, sanford2024transformers}, or put Transformer into various computation models \cite{sanford2024transformers, merrill2023parallelism, peng2024limitations}.
%\cite{sanford2024transformers} also gives unconditional lower bound multi-layer Transform with certain sub-quadratic attention, however with the\lijie{with what?}

The CoT paradigm allows extra computation space for Transformer, and previous work \cite{perez2021attention, feng2023towards, merrill2023expresssive, li2024chain} demonstrated the power of CoT.
\cite{perez2021attention} prove Transformer with arbitrary precise is Turing-complete,
\cite{merrill2023expresssive} and \cite{li2024chain} prove that log-precision Transformer could simulate any polynomial time algorithm/polynomial-size circuit when the algorithm/circuit is given as input.

\newcommand{\TC}{\mathsf{TC}}

\paragraph{Lower bounds against $\TC^0$.} $\TC^0$ is the class of constant-depth circuits consisting of AND, OR, NOT, and linear threshold gates.\footnote{A linear threshold gate $G \colon \{0,1\}^n \to \{0,1\}$ outputs $1$ on input $x \in \{0,1\}^n$ if $\sum_{i=1}^{n} w_i \cdot x_i \ge \theta$ and $0$ otherwise, where $(w_1,\dotsc,w_n) \in \R^{n}$ is a weight vector and $\theta \in \R$ is a threshold.} In the literature, the size of $\TC^0$ circuits is often measured in terms of the number of \emph{wires} (i.e., the total fan-in of all gates).

It is proven in~\cite{merrill2023parallelism} that constant-depth encoder-only transformer can be simulated by $\TC^0$ circuits with roughly the same depth. Hence, strong lower bounds against $\TC^0$ would imply lower bounds against a constant-depth encoder-only transformer. This simulation has at least a quadratic blow-up in terms of wires since it already requires at least an $\Omega(n^2)$-wire $\TC^0$ circuit to simulate a single attention layer. Unfortunately, the best known lower bounds against $\TC^0$ circuits~\cite{ImpagliazzoPS97,ChenS018,HatamiHTT23} are of the form $n^{1+c^{-d}}$ where $c > 1$ is a constant, which is not strong enough to imply any constant-depth encoder-only transformer lower bounds.\footnote{Also, the hard instance from~\cite{ImpagliazzoPS97}, the parity function, can be solved by a decoder-only transformer with $\polylog(n)$ model dimension. Here, the parity function takes $n$ bits as input and outputs $1$ if the number of ones in the input is odd, and $0$ otherwise.} Indeed, proving better lower bounds against $\TC^0$ (say, $n^{1+1/d}$ wires against depth-$d$ $\TC^0$) is a major open question in complexity theory and would have breakthrough consequences~\cite{AllenderK10,ChenT19}.

%\Binghui{Mention that the hard instance parity of \cite{ImpagliazzoPS97} can be solved by Transformer?}

%\Binghui{Do we want to cite some circuit lower bound paper? I think we need to at least cite Russell's paper on threshold circuit, we also need to remark that parity is not a hard instance for transformer}

\section{Preliminaries}
\label{sec:pre}

\paragraph{Notation.} Recall we write $[n] = \{1, 2, \ldots, n\}$ and $[n_1:n_2] = \{n_1, n_1+1, \ldots, n_2\}$. In the corner case, we let $[0] = \emptyset$. 
%For random variables $X, Y$, we use $H(X)$ to denote the entropy of $X$, $I(X; Y)$ to denote the mutual information between $X, Y$, and $\KL(X||Y)$ to denote the KL divergence between $X, Y$.

\subsection{Transformer}\label{sec:pre:transformer}

We formally describe the decoder-only Transformer architecture.
Let $L$ be the number of attention layers, $H$ be the number of attention heads at each layer, $p$ be the precision, $d$ be the key, value and model dimension, $n$ be the (input) prompt length. 
In order to carry out attention, it is common to assume the bit precision $p \geq \log(n)$.

An $L$-layer decoder-only Transformer is a sequence-to-sequence network, consists of alternating attention layer and MLP layer:
\[
f_{\tran} = f^{(L)}_{\mlp} \circ f^{(L)}_{\attn} \circ \cdots \circ f^{(1)}_{\mlp} \circ f^{(1)}_{\attn} 
\]
Given an input sequence $x^{(0)} = (x_1^{(0)}, \ldots, x_n^{(0)}) \in (\R^{dH})^n$, the Transformer inductively computes the output of the $\ell$-th attention layer $y^{(\ell)} = (y_1^{(\ell)}, \ldots, y_n^{(\ell)})$ and the output of the $\ell$-th MLP layer $x^{(\ell)} = (x_1^{(\ell)}, \ldots, x_n^{(\ell)})$. 
For layer $\ell = 1,2, \ldots, L$, 
\begin{itemize}
\item {\bf Attention layer $f_{\attn}^{\ell}$}: For each attention head $h \in [H]$ and position $i \in [n]$, we have
\begin{align}
y^{(\ell, h)}_{i} = \sum_{j \leq i}\alpha_{i, j}^{(\ell, h)}V^{(\ell, h)}x_{j}^{(\ell-1)} \in \R^d \label{eq:attention1}
\end{align}
where $\{\alpha_{i, j}^{(\ell, h)}\}_{j \leq i}$ is the attention probability of the $h$-th attention head, computed as
\begin{align}
\alpha_{i, j}^{(\ell, h)} = \frac{\exp((x_{i}^{(\ell-1)})^{\top}(Q^{(\ell, h)})^{\top} K^{(\ell, h)} x_j^{(\ell-1)})}{\sum_{j\leq i}\exp((x_{i}^{(\ell-1)})^{\top}(Q^{(\ell, h)})^{\top} K^{(\ell, h)} x_j^{(\ell-1)})} \in [0,1]  \label{eq:attention2}
\end{align}
%\hongxun{Is it $x^{(\ell - 1)}_i$ in (3)?}
and $Q^{(\ell, h)}, K^{(\ell, h)}, V^{(\ell, h)} \in \R^{d\times dH}$ is the query, key and value matrix of the attention head.

Finally, the output of the $\ell$-th attention layer is the concatenation of each head, 
\begin{align*}
y_{i}^{(\ell)} = (y_{i}^{(\ell, 1)}, \ldots, y_{i}^{(\ell, H)}) \in \R^{dH} \quad \forall i\in [n] 
\end{align*}
\item {\bf MLP layer $f_{\mlp}^{\ell}$}: The output of the $\ell$-th layer (and also the input to the $(\ell+1)$-th layer) is an arbitrary function $g^{(\ell)}: \R^{dH} \rightarrow \R^{dH}$ applied to each position:
\[
x_{i}^{(\ell)} = g^{(\ell)}(y_{i}^{(\ell)}) \in \R^{dH}
\]
\end{itemize}
Our formulation is general enough to capture detailed architecture choice such as the positional encoding, the residual connection, the layer normalization and the mixture-of-expert layers.

The only different between Transformer encoder and decoder is that the encoder performs attention to all tokens (instead of previous token), so the summation in Eq.~\eqref{eq:attention1}\eqref{eq:attention2} is over $j\in[n]$ (instead of $j \leq i$).
%\Binghui{The definition is a little bit inconsistent with current LMs, now we assumed $K \in \R^{d\times d}$ and $g: \R^{dH}\rightarrow \R^d$ but in practice it should $K \in \R^{d\times dH}$ and $g: \R^{dH}\rightarrow \R^{dH}$?}

\subsection{Sequential function composition}
We use the following parameters throughout the paper.
\begin{align}
K = (HdpL)^8 \cdot { 8^{2L^2}}, \quad m = K^{\sum_{\ell \in [0:L-1]}8^{\ell} + 1}, \quad  n_\ell = K^{4\cdot 8^{L - \ell-1}} \quad \forall \ell \in [L-1]. \label{eq:parameter1}
\end{align}
and 
\begin{align}
N_{\ell} = m \cdot  \prod_{\ell'\in [\ell]} n_{\ell} \quad \forall \ell \in [0: L-1]\label{eq:parameter2}.
\end{align}

\begin{definition}[$L$-sequential function composition]
\label{def:composition}
A $L$-sequential function composition task 
%$f_{z_0, z_1, \ldots, z_{L}}: [n_{1} \times n_2 \times \cdots \times n_{L-1}]\rightarrow [N_{L-1}]$ 
$L\text{-}\SeqComp(w,z_0,z_1,\dotsc,z_{L})$
is described a sequence of functions $z_{0}, z_1 \ldots, z_{L}$, where $z_0 \in [m]$ and $z_{\ell}: [N_{\ell-1} ] \rightarrow [N_{\ell-1}]$ for $\ell \in [L]$ and a query $w = (w_{1}, \ldots, w_{L-1}) \in [n_{1}] \times \cdots \times [n_{L-1}]$, one computes
\begin{align*}
i_{0} = z_0 \in [m], \quad i_1 = z_1(i_0) \in [N_0]
\end{align*}
and one inductively computes, for each $\ell = 1,2,\ldots, L-1$:
\begin{align}
i_{2} = z_2(w_1, i_1) \in [N_1], 
%\quad i_3 = z_3(w_2, i_2) \in [N_2],  
\quad \ldots , \quad i_{\ell+1} = z_{\ell+1}(w_{\ell}, i_{\ell}) \in [N_{\ell}] \label{eq:function}
\end{align}
The final output is taken as $L\text{-}\SeqComp(w,z_0,z_1,\dotsc,z_{L}) = i_{L}.
$

%We also define
%\[
%L\text{-}\SeqComp(w,z_0,z_1,\dotsc,z_{L}) = i_{L}.
%\]
\end{definition}

For Transformer to solve the $L$-sequential function composition, we assume the input prompt first describes $L$ functions in the order of $z_{L-1}, \ldots, z_{0}$, and then describes the query $w$. For simplicity, we assume each entry of $z_{\ell}$ $(\ell\in [0:L-1])$ is described using one token (so it takes $N_{\ell-1}$ tokens to describe $z_{\ell}$); the query $w$ is described in one token. 

%\lijie{We probably want to say a sentence here on how the inputs of the above function is partition into tokens (i.e., each value of the functions takes one token, and the final query takes one token)... This is only implicitly described in Lemma 3.1 and can be quite confusing I think...}

\section{The Autoregressive Communication Model}
\label{sec:reduction}

In this section, we introduce a new communication model, which is a nice abstraction of the Transformer architecture for our purpose of proving lower bounds.

\begin{construction}{The Autoregressive Communication Model}

\paragraph{Settings.} Let $L$ be the number of epochs. There are $L + 2$ players; we call them players $[-1:L]$. They communicate in $L$ epochs. 
%For notational convenience, we also use player $-1$ to denote player $\e$. 
Let $B$ be the message bits of a single token.
$\newline$

\paragraph{Input.} Player $i \in [-1 : L]$ receives $m_{(i)}$ tokens $z_{i}$ as input.
$\newline$

\paragraph{Communication.} The communication proceeds in $L$ epochs. For $\ell \in [0:L]$, let $X_{i}^{(\ell)}$ be the message collected by player $i$ ($\in [-1 : L]$) after the $\ell$-th epoch of communication. Initially, when $\ell = 0$, $X_{i}^{(0)}$ is just the input of player $i$. 

For $\ell = 1,2,\ldots, L$, the $\ell$-th epoch of communication proceeds as follows. For player $i \in [-1 : L]$, 
\begin{itemize}
\item The player $i$ sends its information $X_{i}^{(\ell-1)}$ to all players $[i+1: L]$. 
\item The player $j \in [i+1: L]$, based on its own information $X_{j}^{(\ell-1)}$ and player $i$'s information $X_{i}^{(\ell-1)}$, it sends a message $\Pi_{j, i}^{(\ell)}$ to player $i$. The length of the message satisfies 
\[
|\Pi_{j, i}^{(\ell)}| = 2 B \cdot m_{(i)}.
\]
\item Finally, the player $i$ updates its collection of information as 
\[
X_{i}^{(\ell)}:= X_{i}^{(\ell-1)} \cup \bigcup_{j > i} \Pi_{j, i}^{(\ell)}.
\]
\end{itemize}

\paragraph{Output.} At the end of $L$-th round, the player $-1$ outputs a message based on its information $X_{-1}^{(L)}$. 
%\hongxun{I think this is the first time player $-1$ occurs. Did we properly define it? We always say player $-1$ before this line. Below this line we used a lot of player $e$.}

\end{construction}

%\lijie{if we have time, maybe we can draw a graph here...}

\newcommand{\ARSeqComp}{\mathsf{AR\text{-}SeqComp}}

Below, we state $L$-$\SeqComp$ in the autoregressive communication model formally.

%\lijie{In below, should probably use something else instead of $Hdp$...}

\begin{construction}{$L$-$\SeqComp$ in the autoregressive communication model}
    \paragraph{Settings.}
    Let $L$ be a parameter. Let $K,m, n_\ell, N_\ell$ be defined according to~\eqref{eq:parameter1} and~\eqref{eq:parameter2}. We set $B = Hdp$.
    $\newline$

    \paragraph{Input.}  Player $i \in [L]$ receives $z_i$ from~\Cref{def:composition} as an input, described in $N_{i-1}$ tokens. Player $0$ and player $-1$ receives $z_0$ and $w$ as input, respectively, each described in $1$ token.
    $\newline$

    \paragraph{Output.} At the end of $L$-th round, the player $-1$ needs to output $L$-$\SeqComp(w,z_0,\dotsc,z_{L})$.
\end{construction}
%\lijie{Hmm, I don't see where it's justified that $Hdp$ bits is enough to describe an element from $[N_\ell]$...}

\subsection{Autoregressive communication lower bounds imply Transformer lower bounds}

The following lemma shows that proving a lower bound for the above communication model immediately implies the desired lower bound against decoder-only Transformers.

\begin{lemma}[Reduction from Transformers to autoregressive communication]
\label{lem:reduction}
If there is an $L$-layer decoder-only Transformer that solves the $L$-sequential function composition task, then there is a deterministic autoregressive communication protocol that solves $L$-$\SeqComp$ with $L$ epochs and $B = Hdp$ message bits.
\end{lemma}
\begin{proof}

%Consider the following prompt $x = (z_{L}, z_{L-1}, \ldots, z_{1}, z_0, w)$. The prompt first describes the functions $z_{L}, z_{L-1}, \ldots, z_{1}, z_0$, and then the query $w$. 
%The total length of the prompt satisfies $N_{L-1} + \cdots + N_0 + 1 + 1 \leq n$ when $Hdp \leq n^{2^{-4L}}$, due to the choice of parameters (see Eq.~\eqref{eq:parameter1}\eqref{eq:parameter2}).\lijie{I don't see why...}

Let $\Gamma$ be an $L$-layer Transformer that solves the $L$-sequential function composition task. We use $K^{(\ell, h)}, Q^{(\ell, h)}, V^{(\ell, h)}$ to denote its key, query and value matrix, and $g^{(\ell)}$ to denote its MLP function. For $i \in [-1:L]$, $E_{i} \subseteq [n]$ be the positions correspond to the input of player $i$, i.e., 
\begin{align*}
E_{i} =
\left\{
\begin{matrix}
\left[\sum_{j \in [i+1:L]}N_{j-1} + 1: \sum_{j \in [i:L]} N_{j-1}\right] & i \in [L]\\
\{\sum_{j\in [0:L-1]}N_{j} + 1\} & i = 0\\
\{\sum_{j\in [0:L-1]}N_{j} + 2\} & i = -1\\
\end{matrix}
\right.
\end{align*}

Consider the following autoregressive communication protocol for $L$-$\SeqComp$ with $L$ epochs and $B= Hdp$. For epoch $\ell = 1,2,\ldots, L$, for each player $i$ and for each $j > i$, define the transcript $\Pi_{j, i}^{(\ell)}$ from player $j$ to player $i$ as follow:
\begin{align*}
\Pi_{j, i}^{(\ell)} =&~ \left\{\sum_{t\in E_{j}} \exp((x_{r}^{(\ell-1)})^{\top}(Q^{(\ell, h)})^{\top} K^{(\ell, h)} x_{t}^{(\ell-1)}) V^{(\ell, h)}x_{t}^{(\ell-1)}\right\}_{h\in [H], r\in E_{i}} \\
&~ \bigcup \left\{ \sum_{t\in E_{j}} \exp((x_{r}^{(\ell-1)})^{\top}(Q^{(\ell, h)})^{\top} K^{(\ell, h)} x_{t}^{(\ell-1)})\right\}_{h\in [H], r\in E_{i}}
\end{align*}
The total length of $\Pi_{j, i}^{(\ell)}$ satisfies the requirement, i.e.,  we have
\[
|\Pi_{j, i}^{(\ell)}| = |E_{i}| \cdot H \cdot (dp + p) \leq |E_i| \cdot 2Hdp = \left\{ 
\begin{matrix}
2Hdp \cdot N_{i-1} & i \in [L]\\
2Hdp & i \in \{0, -1\}
\end{matrix}
\right.
\]
%It remains to prove that the player $j$ knows the value of the value of $\{x_{r}^{(\ell-1)}\}_{r \in E_i}$ and $\{x_{t}^{(\ell-1)}\}_{t\in E_j}$ during epoch $\ell$.

We note that $\Pi_{j, i}^{(\ell)}$ depends only on the value of $\{x_{r}^{(\ell-1)}\}_{r \in E_i}$ and $\{x_{t}^{(\ell-1)}\}_{t\in E_j}$, hence, it remains to prove that they can be derived based on $\{X^{(\ell-1)}_{r}\}_{r \in E_j}$ and $\{X^{(\ell-1)}_{t}\}_{t \in E_i}$. To this end, we prove
\begin{claim}
\label{claim:reduction-hypothesis}
For $\ell = 0, 1,\ldots, L$, the player $i$ $(i \in [-1:L])$ knows the intermediate value $\{x^{(\ell)}_{t}\}_{t \in E_i}$ of the Transformer after $\ell$-th epoch, i.e., $\{x^{(\ell)}_{t}\}_{t \in E_i}$ can be derived from $\{X^{(\ell)}_{t}\}_{t \in E_i}$.
\end{claim}
With Claim \ref{claim:reduction-hypothesis} in hand, we can finish the proof since the player $-1$ knows the output embedding $x_{n}^{(L)}$ of the Transformer.
\end{proof}

%We prove Lemma \ref{lem:reduction-hypothesis}
\begin{proof}[Proof of Lemma \ref{claim:reduction-hypothesis}]
The inductive hypothesis is obviously true for $\ell = 0$. Suppose it continues to hold for $\ell-1$, then for the $\ell$-th epoch, the player $i$ receives $\Pi_{j, i}^{(\ell)}$ for $j > i$, recall the output $\{y_{r}^{(\ell)}\}_{r \in E_{i}} =\{y_{r}^{(\ell, h)}\}_{r \in E_{i},h\in [H]}$, and we have 
\begin{align*}
y_{r}^{(\ell, h)} = &~ \sum_{t \leq r}\alpha_{r, t}^{(\ell, h)}V^{(\ell, h)}x_{t}^{(\ell-1)} \\
= &~ \frac{\sum_{j \geq i}\sum_{t\in E_{j}, t\leq i} \exp((x_{r}^{(\ell-1)})^{\top}(Q^{(\ell, h)})^{\top} K^{(\ell, h)} x_{t}^{(\ell-1)}) V^{(\ell, h)}x_{t}^{(\ell-1)}}{\sum_{j \geq i}\sum_{t\in E_{j}, t\leq i}\exp((x_{r}^{(\ell-1)})^{\top}(Q^{(\ell, h)})^{\top} K^{(\ell, h)} x_{t}^{(\ell-1)})}
\end{align*}
It can be obtained from the transcript $\{\Pi_{j, i}^{(\ell)}\}_{j > i}$ and $X_{i}^{(\ell-1)}$. Once the player $i$ obtains $\{y_{r}^{(\ell)}\}_{r \in E_{i}}$, it can also obtain $\{x_{r}^{(\ell)}\}_{r \in E_{i}}$ since $x_r^{(\ell)} = g^{(\ell)}(y_r^{(\ell)})$ can be computed locally. This completes the proof.
\end{proof}

\section{Autoregressive Communication Lower Bound}

Our main result (Theorem \ref{thm:main}) can be obtained from Lemma \ref{lem:reduction} and the following lower bound for the communication problem.

\begin{theorem}
\label{thm:main-cc}
There is no deterministic autoregressive communication protocol solving $L$-$\SeqComp$ with $L$ epochs and $B = Hdp$ message bits.
\end{theorem}

%In this section, we fix an $L$-epoch $Hdp$ message bits autoregressive communication protocol $\Pi$. We will show that $\Pi$ cannot solve $L$-$\SeqComp$. This proves~\Cref{thm:main-cc} since $\Pi$ is arbitrary.

\paragraph{Notation.} For notational convenience, we use $z_{-1}$ and $w$ interchangeably to denote player $-1$'s input. In the following, we elaborate on several key definitions that will be crucial to our proof.

\begin{itemize}
    \item (The transcript $\Pi^{(\ell)}_{j,i}$) For any $i \in [-1: L-1]$, $j\in [i+1: L]$, $\ell\in [L]$, recall $\Pi_{j, i}^{(\ell)}$ is the transcript sent from the player $j$ to the player $i$ at the $\ell$-th epoch of communication. Its value is determined by the input of players $[i: L]$, i.e., $z_{L}, \ldots, z_{i}$, and its value is independent of the choice of $z_{i-1}, \ldots, z_{0}, w$.\footnote{The transcript also depends on the underlying autoregressive communication protocol $\Pi$, which will be clear from the context.}
        
    For any fixed value $\wt{z}_{L} \in [N_{L-1}]^{N_{L-1}}, \ldots, \wt{z}_{i}\in [N_{i-1}]^{N_{i-1}}$, let $\Pi_{j, i}^{(\ell)}(\wt{z}_{L}, \ldots, \wt{z}_{i})$ be the transcript when the players $t$ receives input $z_t = \wt{z}_{t}$ ($t\in [i:L]$). 

    \item (The partial composition value $i_{\ell}(\wt{w}, \wt{z}_0, \ldots, \wt{z}_{\ell})$) 
For any $\ell \in [0: L]$, the value of $i_{\ell}$ is determined by $w, z_0, \ldots, z_{\ell}$. We write $i_{\ell}(\wt{w}, \wt{z}_0, \ldots, \wt{z}_{\ell})$ to denote the value of $i_{\ell}$ when $w = \wt{w}, z_0 = \wt{z}_0, \ldots, z_{\ell} = \wt{z}_{\ell}$.
\end{itemize}

%For any $i' \in [i: L]$, let $\Pi_{j, i}^{(\ell)}(\wt{z}_{i'}, \ldots, \wt{z}_{i})$ be the transcript when the party $t$ receives $\wt{z}_{t}$ as input ($t\in [i: i']$), so $\Pi_{j, i}^{(\ell)}(\wt{z}_{i'}, \ldots, \wt{z}_{i})$ is a random variable whose value depends on $z_{L}, \ldots, z_{i'+1}$. 
%For any subsets $\{Z_{t}\}_{t\in [i: i']}$, where $Z_{t} \subseteq [N_{t-1}]^{N_{t-1}}$ ($t \in [i: i']$), define
%\begin{align*}
%\Pi_{j, i}^{(\ell)}(\times_{t \in [i: i']} Z_{t}) = \bigcup_{\wt{z}_{i}\in Z_i, \ldots, \wt{z}_{i'} \in Z_{i'}}\Pi_{j, i}^{(\ell)}(\wt{z}_{i'}, \ldots, \wt{z}_{i})
%\end{align*}
%In words, $\Pi_{j, i}^{(\ell)}(\times_{t \in [i: i']} Z_{t})$ is the collection of transcripts when $z_{i}, \ldots, z_{i'}$ could take value from $Z_{i}, \ldots, Z_{i'}$. Again, it is a random variable whose value is determined by $z_{L}, \ldots, z_{i'+1}$.

\paragraph{Parameters.} We use the following parameters 
\begin{align}
x_\ell = K^{8^{L- \ell-1}} (\forall \ell \in [0:L-1]), \quad A_{\ell} = \left[N_{\ell-1}^{N_{\ell-1}}\right] \quad (\forall \ell \in [L]) \label{eq:parameter3}
\end{align}
and
\begin{align}
\Delta_{\ell} = 2^{{4\sqrt{K}} (x_0\ldots x_{\ell-2})\cdot (n_1\ldots n_{L-1}) } \quad (\forall \ell \in [2:L]), \quad  \Theta_{\ell} = 8^{-L \ell}(x_{0} \ldots x_{\ell})\cdot (n_1\ldots n_{\ell-1})  \quad (\forall \ell \in [L-1]) \label{eq:def-delta}.
\end{align}

For notational convenience, we also set $A_{-1} = \prod_{i=1}^{L-1}[n_i]$ and $A_0 = [m]$. Note that with our convention of denoting $w$ by $z_{-1}$, player $i$ takes an input from $A_{i}$ for every $i \in [-1 : L]$.\footnote{Here, we can interpret an element from $\left[N_{\ell-1}^{N_{\ell-1}}\right]$ as a function from $[N_{\ell-1}] \to [N_{\ell-1}]$ via a natural bijection.}

The meaning of parameters will be clearer after we state our main technical centerpiece Lemma \ref{lem:main}.

%We take an inductive approach to prove the communication lower bound. The main step is to prove the following Lemma.

\paragraph{Indistinguishable decomposition.} Our key idea is the following concept of \emph{indistinguishable decomposition}, which is two sets $R_{\ge \ell}$ and $Z_{<\ell}$, where $R_{\ge \ell}$ is a set of input assignments to players $[\ell:L]$ and $Z_{<\ell}$ is a set of input assignments to players $[-1:\ell-1]$), such that for every possible inputs $z_{<\ell} \in Z_{<\ell}$, all assignments to $R_{\ge \ell}$ are indistinguishable to players $[-1:\ell-1]$ on inputs $z_{< \ell}$ after $\ell$ epochs (because they lead to the same transcripts).

Formally, we define:

\begin{definition}[Indistinguishable decomposition]
    Let $\ell \in [2:L]$, 
    \[
    R_{\ge \ell} \subseteq A_L \times A_{L-1} \times \cdots \times A_{\ell}
    \]
    and 
    \[
    Z_{< \ell} = Z_{-1} \times \cdots \times Z_{\ell-1} \subseteq A_{-1}\times \cdots \times A_{\ell-1} \quad \text{where} \quad Z_{-1}= A_{-1}, Z_0 \subseteq A_0, \cdots , Z_{\ell-1}\subseteq A_{\ell-1}.
    \]
    We say $R_{\ge \ell}$ and $Z_{< \ell}$ is an indistinguishable decomposition, if for every $\wt{z}_{<\ell} \in Z_{< \ell}$, and for every $\wt{\alpha}_{\ge \ell}, \wt{\beta}_{\ge \ell} \in R_{\ge \ell}$, it satisfies:
    \[
        \Pi_{j,i}^{(\ell')}(\wt{z}_{<\ell},\wt{\alpha}_{\ge \ell}) = 
        \Pi_{j,i}^{(\ell')}(\wt{z}_{<\ell},\wt{\beta}_{\ge \ell})
    \]
    for every $j \in [\ell:L]$, $i \in [-1:\ell-1]$, and $\ell' \in [\ell]$.
\end{definition}
%Note that below, we have the additional structure that $Z_{<L}$ is a product set. 

Indistinguishable configuration is helpful because when $\ell = L$, for every input assignment from $Z_{<L}$ to players $[-1:L-1]$, the player $-1$ after $L$ epochs (i.e., at the end of the protocol) sees the same transcript when player $L$ receives inputs from $R_{\ge L}$. In particular, it means for every $\wt{z}_{<L} \in Z_{< L}$, the answer $L$-$\SeqComp(\wt{z}_{<L}, \wt{z}_L)$ must be the same for every $\wt{z}_L \in R_{\ge L}$. This is a strong requirement and we will carefully define properties of $R_{\ge \ell}$ and $Z_{< \ell}$ such that this requirement would lead to contradiction, and therefore obtaining our lower bound.

For a subset $Z_{<\ell}$, we define $\mathcal{I}_{\ell-1}$ as
\begin{align}
\mI_{\ell-1}(Z_{<\ell}):= \{\wt{i}_{\ell-1}: \wt{i}_{\ell-1} = i_{\ell-1}(\wt{z}_{-1}, \wt{z}_0, \ldots, \wt{z}_{\ell-1}) \text{ for some } (\wt{z}_{-1}, \wt{z}_0, \dotsc, \wt{z}_{\ell-1}) \in Z_{<\ell} \}. \notag
\end{align}

That is, $\mI_{\ell-1}(Z_{<\ell})$ is the set of all partial composition values for inputs from $Z_{<\ell}$.

The following lemma shows that the desired lower bound follows from a good enough indistinguishable configuration for $\ell = L$.

\begin{lemma}\label{lemma:id-suffices}
    Let $\Pi$ be an $L$-epoch $Hdp$ message bits autoregressive communication protocol. If there is an indistinguishable decomposition $R_{\ge L}$ and $Z_{< L}$ such that:
    \begin{enumerate}
        \item (\textbf{Large remaining entropy}) $|R_{\ge L}| \ge |A_L|/ \Delta_L$.
        \item (\textbf{Large cover}) $|\mathcal{I}_{L-1}(Z_{<L})| \ge \Theta_{L-1}$.
    \end{enumerate}
    Then $\Pi$ does not solve $L$-$\SeqComp$.
\end{lemma}

In the next subsection, we will show the existence of the required indistinguishable decomposition from~\Cref{lemma:id-suffices} via an induction, which finishes the proof of~\Cref{thm:main-cc}.

\begin{lemma}\label{lemma:id-exists}
    For every $L$-epoch $Hdp$ message bits autoregressive communication protocol $\Pi$, there is an indistinguishable decomposition $R_{\ge L}$ and $Z_{< L}$ satisfying the requirements of~\Cref{lemma:id-suffices}.
\end{lemma}

We finish this subsection by proving Lemma~\ref{lemma:id-suffices}.

\begin{proof}[Proof of Lemma \ref{lemma:id-suffices}]
%We take $\ell = L$ in Lemma \ref{lem:main} and obtain sets $Z_0 \subseteq [A_{0}] \ldots, Z_{L-1} \subseteq [A_{L-1}]$ and transcripts $\Lambda^{(L)}$. The set $R_{\ge L}\subseteq [A_{L}]$ contains all $\wt{z}_{L}\in [A_{L}]$ that is consistent with $\Lambda^{(L)}$, and by the first property of Lemma \ref{lem:main} (i.e., Eq.~\eqref{eq:inductive-hypothesis1}), its size is at least 
%\lijie{TODO: fix the below.}
%We prove by contradiction and assuming the existence of $L$ epochs $Hdp$ messages bits autoregressive communication protocol that solves $L$-$\SeqComp$.\lijie{no need to do proof-by-contradiction now?}

First, by the large remain entropy property and our choice of parameters, we have
\begin{align*}
|R_{\geq L}| \geq &~ |A_{L}|/\Delta_{L} = |A_{L}|\cdot 2^{-{4\sqrt{K}} x_0 \cdots x_{L-2} \cdot n_{1}\cdots n_{L-1}} > |A_{L}|\cdot 2^{-8^{-L^2}(x_0\cdots x_{L-1})(n_1\cdots n_{L-1})}\\
\geq &~ |A_{L}|\cdot 2^{-n_{L-1}\Theta_{L-1}}
\geq |A_{L}| \cdot 2^{-n_{L-1}|\mI_{L-1}(Z_{<L})|}
\geq \frac{|A_{L}|}{ (N_{L-1})^{n_{L-1}|\mI_{L-1}(Z_{<L})|}}.
%|A_{L}|\cdot 2^{-n_{L-1}\Theta_{L-1}} \leq |A_{L}|\cdot 2^{-8^{-L^2}(x_0\cdots x_{L-1})(n_1\cdots n_{L-1})} < |A_{L}|\cdot 2^{-K x_0 \cdots x_{L-2} \cdot n_{1}\cdots n_{L-1}} = |A_{L}|/\Delta_{L}
\end{align*}
%\hongxun{I don't understand the second inequality. Why is $|A_{L}|\cdot 2^{-K x_0 \cdots x_{L-2} \cdot n_{1}\cdots n_{L-1}} > |A_{L}|\cdot 2^{-8^{-L^2}(x_0\cdots x_{L-1})(n_1\cdots n_{L-1})}$? Where does $2^{-8^{-L^2}}$ come from? See group chat.}
Here the second step follows from the definition of $\Delta_{L}$ (see Eq.~\eqref{eq:def-delta}), the third and fourth step follow from the choice of parameters (see Eq.~\eqref{eq:parameter1}\eqref{eq:parameter3}), the fifth step follows from $|\mI_{\ell-1}(Z_{<L})| \geq \Theta_{\ell-1}$, i.e., the large cover property.
%the last step follows from the definition of $\Delta_{L}$ (see Eq.~\eqref{eq:def-delta}).

 Hence, there exists $\wt{i}_{L-1} \in \mI_{L-1}(Z_{<L})$ and $\wt{w}_{L-1} \in [n_{\ell-1}]$, such that, there exists $z_{L}', z_{L}''\in R_{\ge L}$ satisfying $z_{L}'(\wt{w}_{L-1}, \wt{i}_{L-1}) \neq z''_{L}(\wt{w}_{L-1}, \wt{i}_{L-1})$. 
 %For $\wt{i}_{L-1} \in \mI_{L-1}(Z_{<L})$,  
 By definition, there exists $\wt{z}_{<L}\in Z_{<L}$, 
such that (1) $\wt{i}_{L-1} = i_{L-1}(\wt{z}_{<L})$ and (2) $\wt{z}_{-1, L-1} = \wt{w}_{L-1}$.
 %For each $\wt{w}_{L-1} \in [n_{L-1}]$, 
 %We first claim that 
 %$z_{L}(\wt{w}_{L-1}, \wt{i}_{L-1}) \in [N_{\ell-1}]$ is the same for every $z_{L} \in R_{\ge L}$. 
 %On the contrary, suppose there are $z_{L}', z_{L}''\in R_{\ge L}$ such that $z_{L}'(\wt{w}_{L-1}, \wt{i}_{L-1}) \neq z''_{L}(\wt{w}_{L-1}, \wt{i}_{L-1})$. 
 By the definition of $R_{\ge L}$, the messages sent from player $L$ to players $[-1:L-1]$ are the same under $z_{L}', z_{L}''$, when the $\ell$-th player ($\ell\in [-1:L-1]$) receives $\wt{z}_{\ell}$ as input.
 Hence, the player $-1$ can not distinguish between $z_{L}''$ and $z_{L}'$, and therefore, it could not be correct on answering both $i_{L}' = z_{\ell}'(\wt{w}_{L-1}, \wt{i}_{L-1})$ and $i_{L}'' = z_{\ell}''(\wt{w}_{L-1}, \wt{i}_{L-1})$, since $z_{L}'(\wt{w}_{L-1}, \wt{i}_{L-1}) \neq z''_{L}(\wt{w}_{L-1}, \wt{i}_{L-1})$.

%We have proved that, for any $\wt{w}_{L-1}\in [n_{L-1}]$ and $\wt{i}_{L-1} \in \mI_{L-1}$, $z_{L}(\wt{w}_{L-1}, \wt{i}_{L-1})$ is the same for every $z_{L} \in R_{\ge L}$. This means the size of $R_{\ge L}$ is at most
%\begin{align*}
%|R_{\ge L}| \leq \frac{|A_{L}|}{ (N_{L-1})^{n_{L-1}|\mI_{L-1}|}} \leq |A_{L}| \cdot 2^{-n_{L-1}|\mI_{L-1}|} \leq  |A_{L}| \cdot 2^{-n_{L-1}\Theta_{L-1}},
%\end{align*}
%where the last step follows from $|\mI_{\ell-1}| \geq \Theta_{\ell-1}$ (see Eq.~\eqref{eq:inductive-hypothesis2}).

%This contradicts with Lemma \ref{lemma:id-suffices}, since
%\[
%|A_{L}|\cdot 2^{-n_{L-1}\Theta_{L-1}} \leq |A_{L}|\cdot 2^{-8^{-L^2}(x_0\cdots x_{L-1})(n_1\cdots n_{L-1})} < |A_{L}|\cdot 2^{-K x_0 \cdots x_{L-2} \cdot n_{1}\cdots n_{L-1}} = |A_{L}|/\Delta_{L}
%\]
%where the first and the second step follows from the choice of parameters (see Eq.~\eqref{eq:parameter1}\eqref{eq:parameter3}), the last step follows from the definition of $\Delta_{L}$ (see Eq.~\eqref{eq:def-delta}). We complete the proof here.
\end{proof}

\subsection{Constructing Indistinguishable Decompositions via Induction}

The rest of this section is devoted to prove Lemma \ref{lemma:id-exists}, we will indeed prove it via an induction specified in the~\Cref{lem:main}, whose $\ell = L$ case is exactly Lemma~\ref{lemma:id-exists}.

%We prove the base case $\ell=2$ in Section \ref{sec:initial} and we finish the inductive step in Section \ref{sec:induction}. The proof of technical Lemma appears at Section \ref{sec:tech}.

\begin{lemma}[Main Lemma]
\label{lem:main}
For any $\ell \in [2: L]$, 
\begin{itemize}
\item We have a pair of sets $(R_{\geq \ell}, Z_{<\ell})$, where $R_{\geq \ell} \subseteq A_{L}\times A_{L-1}\times \cdots \times A_{\ell}$, $Z_{< \ell} = Z_{-1}\times Z_{0}\times \cdots \times Z_{\ell-1}$, with $Z_{-1} = [n_1\cdots n_{L-1}]$, $Z_0 \subseteq A_0$, $Z_{1} \subseteq A_1, \ldots, Z_{\ell-1} \subseteq A_{\ell-1}$ and they have size $|Z_0| = x_0, |Z_1| = x_1, \ldots, |Z_{\ell-1}| = x_{\ell-1}$; 
\item We can fix the transcript from players $[\ell: L]$ to $[-1:\ell-1]$ at the first $\ell$ epochs, when the players $[-1:\ell-1]$ take input from $Z_{<\ell}$.
i.e., %\footnote{Slight abuse of notation, we would think of $\e = -1$ and $Z_{-1} = W$.}
\begin{align*}
\Lambda^{(\ell)} := &~  \left(\Lambda_{j, i}^{(\ell, \ell')}\right)_{j \in [\ell: L], i \in [-1:\ell-1], \ell' \in [\ell]}
\end{align*}
where
\begin{align*}
\Lambda_{j, i}^{(\ell, \ell')} := \left(\Lambda_{j, i}^{(\ell, \ell')}(\wt{z}_{\ell-1}, \ldots, \wt{z}_{i})\right)_{\wt{z}_{\ell-1} \in Z_{\ell-1}, \ldots, \wt{z}_{i} \in Z_i} \quad \text{and} \quad \Lambda_{j, i}^{(\ell, \ell')}(\wt{z}_{\ell-1}, \ldots, \wt{z}_{i}) \in \mathsf{domain}(\Pi_{j, i}^{(\ell')})
\end{align*}
%\hongxun{Is this $\mathsf{domain}(\Pi_{j, i}^{(\ell')})$?}
%\begin{align*}
%\Lambda_{j, i}^{(\ell, \ell')} \in \mathsf{domain}(\Pi_{j, i}^{(\ell')} (\times_{t \in [\ell-1, i]}Z_{t}) )   \quad \quad \forall j \in [\ell: L], i \in \{\e\}\cup [0:\ell-1], \ell' \in [\ell]
%\end{align*}
\end{itemize}
such that we have the following guarantees:
\begin{itemize}
\item ({\bf Consistency}) 
%$R_{\geq \ell}$ is indistinguishable to players $[-1:\ell]$ when they take input assignments from $Z_{< \ell}$, i.e.
$\Lambda^{(\ell)}$ is the first $\ell$-epoch transcript from players $[\ell: L]$ to $[-1:\ell-1]$, %\hongxun{$[-1 : \ell - 1]$?}, 
when they take input from $R_{\geq \ell}$ and $Z_{< \ell}$, i.e.,
\begin{align*}
&~\Pi_{j, i}^{(\ell')}(\wt{z}_{L}, \ldots, \wt{z}_{i}) =  \Lambda_{j, i}^{(\ell, \ell')}(\wt{z}_{\ell-1}, \ldots, \wt{z}_{i}) \\
&~ \forall j \in [\ell: L], i \in [-1:\ell-1], \ell' \in [\ell], \wt{z}_{\geq \ell} \in R_{\geq L}, \wt{z}_{\ell-1}\in Z_{\ell-1}, \ldots \wt{z}_{i}\in Z_i, 
\end{align*}
\iffalse
\begin{align}
\Pi_{j, i}^{(\ell')}(\wt{z}_{L}, \ldots, \wt{z}_{i}) = \Lambda_{j, i}^{(\ell, \ell')}(\wt{z}_{\ell-1}, \ldots, \wt{z}_{i}) 
R_{\geq \ell} \subseteq \left\{ 
\begin{matrix}
\wt{z}_L, \ldots, \wt{z}_{\ell}: \text{such that }\Pi_{j, i}^{(\ell')}(\wt{z}_{L}, \ldots, \wt{z}_{i}) = \Lambda_{j, i}^{(\ell, \ell')}(\wt{z}_{\ell-1}, \ldots, \wt{z}_{i}) \\
\forall j \in [\ell: L], i \in [-1:\ell-1], \ell' \in [\ell], \wt{z}_{\ell-1}\in \wt{Z}_{\ell-1}, \ldots \wt{z}_{i}\in Z_i
\end{matrix} 
\right\}\subseteq [A_{L}\times \cdots \times A_{\ell}] \notag
\end{align}
\fi
\item ({\bf Large remaining entropy}) 
%The total number of possible $z_{L}, \ldots, z_{\ell}$ that are consistent with the transcripts $\Lambda^{(\ell)}$ is large, i.e.,
%and its size satisfies
The size of $R_{\geq \ell}$ is large, i.e.,
\begin{align}
%R_{\geq \ell} := &~ \left\{ \wt{z}_L, \ldots, \wt{z}_{\ell} | \Pi_{j, i}^{(\ell')} (\times_{t' \in [\ell:L]} \wt{z}_{t'} \times_{t \in [\ell-1, i]}Z_{t}) = \Lambda_{j, i}^{(\ell, \ell')} \, \forall j \in [\ell: L], i \in \{\e\}\cup [0:\ell-1], \ell' \in [\ell-1] \right\} \notag\\
|R_{\geq \ell}| \geq &~ |A_{L}|  \cdots  |A_{\ell}| / \Delta_{\ell}  \label{eq:inductive-hypothesis1}.
\end{align}
\item ({\bf Large cover}) The total number of possible $i_{\ell-1}$ under $Z_{-1}, Z_0, Z_1, \ldots, Z_{\ell-1}$ is large, i.e.,
\begin{align}
\mI_{\ell-1}:= \{\wt{i}_{\ell-1}: \wt{i}_{\ell-1} = i_{\ell-1}(\wt{w}, \wt{z}_0, \ldots, \wt{z}_{\ell-1}) \text{ for some } \wt{w} \in Z_{-1}, \wt{z}_0\in Z_0, \ldots, \wt{z}_{\ell}\in Z_{\ell-1}\} \notag
\end{align}
and its size satisfies
\begin{align}
|\mI_{\ell-1}| \geq \Theta_{\ell-1} \label{eq:inductive-hypothesis2}.
\end{align}
\end{itemize}
\end{lemma}

\subsection{The Initial step} 
\label{sec:initial}
We first prove the correctness of Lemma \ref{lem:main} for $\ell = 2$.

\subsubsection{Step 1: Choosing $Z_{0}, Z_{1}$}
We take $Z_0 = [x_0]$ and our first step is to select the set $Z_1 \subseteq [N_0]$.
To this end, consider all possible first epoch messages from the player $1$ to the player $-1$, i.e., 
\[
\Psi_{1,-1}^{(1)} = \left(\Psi_{1,-1}^{(1)}(\wt{z}_{-1})\right)_{\wt{z}_{-1} \in Z_{-1}} \quad \text{where}  \quad \Psi_{1,-1}^{(1)}(\wt{z}_{-1}) \in \{0,1\}^{2Hdp}.
\]
The total number of possible $\Psi_{1,-1}^{(1)}$ is $2^{2Hdp \cdot |Z_{-1}|} = 2^{2Hdp (n_1\cdots n_{L-1})}$, and therefore, there exists one $\wt{\Psi}_{1,-1}^{(1)} \in \{0,1\}^{2Hdp \cdot (n_1\cdots n_{L-1})}$, such that 
\begin{align*}
S := \{\wt{z}_{1} \in A_1: \Pi_{1, -1}^{(1)}(\wt{z}_1, \wt{z}_{-1}) = \Psi_{1,-1}^{(1)}(\wt{z}_{-1}) \, \,  \forall \wt{z}_{-1} \in Z_{-1}\}  \subseteq A_1 
\end{align*}
and
\begin{align}
|S| \geq |A_1| \cdot 2^{-2Hdp \cdot (n_1\cdots n_{L-1})}.\label{eq:initial1}
\end{align}
Note the first epoch message depends only on $\wt{z}_1$ and $\wt{z}_{-1}$ so we can write it as $\Pi_{1, -1}^{(1)}(\wt{z}_1, \wt{z}_{-1})$. 
%\hongxun{I don't immediately see why we need to fix $\Pi^{(1)}_{1,-1}$ and ensure $S$ is indistinguishable for $Z_{-1}$. We did not use it at all in proof of Lemma 4.6. I think we used it in Step 2.1 while determining $C_2$. Maybe mention this there?}

%First, we This implies that

The proof of the following Lemma can be found at Section \ref{sec:tech}.
%\hongxun{This does not imply Lemma 4.6, otherwise we would not have to prove it in section 4.4.}

%We have the following Lemma, whose proof is deferred to Section \ref{sec:tech}.
\begin{lemma}
\label{lem:initial}
There exists a subset $Z_1 \subseteq S$ with size $|Z_1| = x_1$, such that it satisfies  
\begin{align}
|\{ \wt{z}_1(i_0): \wt{z}_1 \in Z_1, i_0 \in Z_0 \}| \geq 8^{-L} x_0 x_1 = \Theta_1. \label{eq:initial-requirement}
\end{align}
\end{lemma}

We take the subset $Z_1$ in Lemma \ref{lem:initial} and it remains to fix the transcripts from players $j \in [2:L-1]$ to players $i = -1, 0, 1$ at the first two epochs.

\subsubsection{Step 2.1: Fixing the transcript to player $-1$} We first fix the transcript to player $-1$. 
For the first epoch, we need to fix $\Lambda^{(2, 1)}_{j, -1}(\wt{z}_1, \wt{z}_0, \wt{z}_{-1})$ for every $\wt{z}_1\in Z_1, \wt{z}_0 \in Z_{0}, \wt{z}_{-1}\in Z_{-1}$. We note that, the first epoch message from player $j$ to player $-1$ depends only on $\wt{z}_{-1}$ and player $j$'s input, but not on $\wt{z}_1, \wt{z}_0$, hence it suffices to find some 
\begin{align*}
\Phi_{j, -1}^{(1)} = \left(\Phi_{j, -1}^{(1)}(\wt{z}_{-1})\right)_{\wt{z}_{-1}\in Z_{-1}}   \quad \text{where} \quad \Phi_{j, -1}^{(1)}(\wt{z}_{-1})\in \{0,1\}^{2Hdp}.
\end{align*}
and set
\[
\Lambda^{(2, 1)}_{j, -1}(\wt{z}_1, \wt{z}_0, \wt{z}_{-1}) = \Phi_{j, -1}^{(1)}(\wt{z}_{-1}) \quad \forall \wt{z}_1\in Z_1, \wt{z}_0 \in Z_0, \wt{z}_{-1} \in Z_{-1}
\] 
The total number of such transcripts are at most $2^{2Hdp \cdot |Z_{-1}|} = 2^{2Hdp \cdot (n_1\cdots n_{L-1})}$.
Hence, we can choose $\{\Lambda^{(2, 1)}_{j, -1}\}_{j\in [2:L]}$ %\hongxun{$[2 : L]?$ Also in the equation below.}
, such that the set of consistent $z_{L},\ldots, z_{2}$, 
\begin{align*}
C_1: = \left\{
\begin{matrix}
(\wt{z}_{L},\ldots, \wt{z}_2) \in A_{L} \times \cdots \times A_2:\\
\Pi_{j, -1}^{(1)}(\wt{z}_{L}, \ldots, \wt{z}_{0}, \wt{z}_{-1}) = \Lambda^{(2, 1)}_{j, -1}(\wt{z}_1, \wt{z}_0, \wt{z}_{-1}) \,\, \forall \wt{z}_1\in Z_1, \wt{z}_0\in Z_0, \wt{z}_{-1}\in Z_{-1}, j\in [2:L] 
\end{matrix}\right\}.
\end{align*}
satisfies
\begin{align*}
|C_1| \geq |A_{L}|\cdots |A_2| \cdot 2^{-2Hdp \cdot (n_1\cdots n_{L-1}) \cdot L}.
\end{align*}

For the second epoch, we need to fix $\Lambda^{(2, 2)}_{j, -1}(\wt{z}_1, \wt{z}_0, \wt{z}_{-1})$ for every $\wt{z}_1\in Z_1, \wt{z}_0 \in Z_{0}, \wt{z}_{-1}\in Z_{-1}$. 
The transcript from player $j \in [2:L]$ to player $-1$ depends only on the information state $X_{-1}^{(1)}$ and $X_{j}^{(1)}$, which are independent of the choice of $z_1 \in Z_1$.
This is because, the only message in $X_{-1}^{(1)}$ and $X_{j}^{(1)}$ that depends on $z_{1}$ is the first epoch message from player $1$ to $-1$, which equals to $\Psi_{1, -1}^{(1)}(\wt{z}_{-1})$ (see Eq.~\eqref{eq:initial1} and Lemma \ref{lem:initial}) and it is the same for every $\wt{z}_1 \in Z_1$. 
Hence it suffices to find some 
\[
\Phi_{j, -1}^{(2)} := \left(\Phi_{j, -1}^{(2)}(\wt{z}_0, \wt{z}_{-1})\right)_{\wt{z}_0\in Z_0, \wt{z}_{-1} \in Z_{-1}}
\]
and set 
\[
\Lambda^{(2, 2)}_{j, -1}(\wt{z}_1, \wt{z}_0, \wt{z}_{-1}) = \Phi_{j, -1}^{(2)}(\wt{z}_0, \wt{z}_{-1}) \quad \forall \wt{z}_1\in Z_1, \wt{z}_0\in Z_0, \wt{z}_{-1} \in Z_{-1}
\]
The total number of such transcripts are at most $2^{2Hdp \cdot x_0 \cdot (n_1\cdots n_{L-1})}$.
Hence, we can properly choose $\{\Lambda^{(2, 2)}_{j, -1}\}_{j\in [2:L-1]}$, such that the set of consistent $(z_{L},\ldots, z_{2})$
\begin{align*}
C_2: = \left\{
\begin{matrix}
(\wt{z}_{L},\ldots, \wt{z}_2) \in C_1:\\
\Pi_{j, -1}^{(2)}(\wt{z}_{L}, \ldots, \wt{z}_{0}, \wt{z}_{-1}) = \Lambda^{(2, 2)}_{j, -1}(\wt{z}_1, \wt{z}_0, \wt{z}_{-1}) \,\, \forall \wt{z}_1\in Z_1, \wt{z}_0\in Z_0, \wt{z}_{-1}\in Z_{-1}, j\in [2:L] 
\end{matrix}\right\}.
\end{align*}
satisfies
\[
|C_2| = |C_1| \cdot 2^{- 2Hdp \cdot x_0 (n_1\cdots n_{L-1})  \cdot L} \geq |A_{L}\cdots A_2| \cdot 2^{-4HdpL \cdot x_0 (n_1\cdots n_{L-1})}
\]

\subsubsection{Step 2.2: Fixing the transcript to player $0$} We then fix the transcript to player $0$. 
The total number of transcripts $\{\Lambda_{j, 0}^{(2, \ell')}\}_{j \in [2:L], \ell'\in [2]}$ of the first two epochs is at most $2^{2Hdp \cdot x_0 x_1 \cdot 2L}$. 
We can fix its value so that the set of consistent $(z_{L},\ldots, z_{2})$  
\begin{align*}
C_3: = \left\{
\begin{matrix}
(\wt{z}_{L},\ldots, \wt{z}_2) \in C_2:\\
\Pi_{j, 0}^{(\ell')}(\wt{z}_{L}, \ldots, \wt{z}_{0}) = \Lambda^{(2, \ell')}_{j, 0}(\wt{z}_1, \wt{z}_0) \,\, \forall \wt{z}_1\in Z_1, \wt{z}_0\in Z_0, j\in [2:L], \ell'\in [2] 
\end{matrix}\right\}.
\end{align*}
satisfies
\[
|C_3| \geq |C_2| \cdot 2^{- 2Hdp \cdot x_0 x_1  \cdot 2L} \geq |A_{L}\cdots A_2| \cdot 2^{-6HdpL \cdot x_0 (n_1\cdots n_{L-1})}.
\]

\subsubsection{Step 2.3: Fixing the transcript to player $1$} Finally, we fix the transcript to player $1$.  
The total number of transcripts 
$\{\Lambda_{j, 1}^{(2, \ell')}\}_{j \in [2:L], \ell'\in [2]}$
of the first two epochs is at most $2^{2Hdp m \cdot x_1 \cdot 2L}$, and we can fix the value so that the set of consistent $(z_{L},\ldots, z_{2})$ 
\begin{align*}
C_4: = \left\{
\begin{matrix}
(\wt{z}_{L},\ldots, \wt{z}_2) \in C_3:\\
\Pi_{j, 1}^{(\ell')}(\wt{z}_{L}, \ldots, \wt{z}_{1}) = \Lambda^{(2, \ell')}_{j, 1}(\wt{z}_1) \,\, \forall \wt{z}_1\in Z_1,  j\in [2:L], \ell'\in [2] 
\end{matrix}\right\}.
\end{align*}
and we have
\begin{align}
C_4 \geq C_3 \cdot 2^{-2Hdp m \cdot x_1 \cdot 2L} \geq &~ |A_{L}|\cdots |A_2| \cdot 2^{-8HdpL \cdot x_0 (n_1\cdots n_{L-1})}\notag \\
\geq &~ |A_{L}|\cdots |A_2| \cdot 2^{-{4\sqrt{K}} x_0 (n_1\cdots n_{L-1})} = |A_{L}|\cdots| A_2| /\Delta_2 \label{eq:initial2}
\end{align}
Here the second step follows from the choice of parameters (see Eq.~\eqref{eq:parameter1}\eqref{eq:parameter3}) and the last step follows from the definition of $\Delta_2$ (see Eq.~\eqref{eq:def-delta}).

Combining Lemma \ref{lem:initial} and Eq.~\eqref{eq:initial2}, we conclude the proof for the case $\ell=2$.

\subsection{Inductive step} 
\label{sec:induction}
Suppose~\cref{lem:main} holds up to $\ell \in [2 : L-1]$, we prove it continues to hold for $\ell+1$. This proceeds in a few steps:
\begin{itemize}
\item We first select the set $Z_{\ell}$.
To this end, we find a rectangular subset from $R_{\ge\ell}$, i.e., we find $S_{\ge \ell+1} \times Z_{\ell} \subseteq R_{\ge \ell}$ where $S_{\ge \ell+1} \subseteq A_L \times \cdots \times A_{\ell+1}$ and $Z_{\ell} \subseteq A_\ell$, with the requirement that (1) $S_{\ge \ell+1}$ has large size, and (2) the size of $\mI(Z_{<\ell+1})$ is large.
see Section \ref{sec:choose_z_l} for details.
\item We then fix the transcripts to players $[-1:\ell-1]$ in the first $\ell$ epochs, for which we simply use the transcript $\Lambda^{(\ell)}$ from the inductive hypothesis; see Section \ref{sec:inductive-step1} for details.
\item Next, we fix the transcripts to players $[-1:\ell-1]$ for the $(\ell+1)$-th epoch. This is the key step of our proof.  Our key insight is that $Z_{\ell}$ is indistinguishable to players $[-1:\ell-1]$ after $\ell$ epochs, when they take input from $Z_{< \ell}$. Hence, the $(\ell+1)$-th epoch transcripts to players $[-1:\ell-1]$ are independent of the choice of $z_{\ell}\in Z_{\ell}$. We can then use a greedy strategy to select transcripts that leak the least amount of information, without consulting to the value of $z_{\ell} \in Z_{\ell}$; see Section \ref{sec:inductive-step2} for details.
\item Finally, we fix the transcripts to players $\ell$. We use a greedy strategy and select the transcript that leaks the least amount of information; see Section \ref{sec:inductive-step3} for details.
\end{itemize}

 Our key insight is that $Z_{\ell}$ is indistinguishable to players $[-1:\ell-1]$ after $\ell$ epochs, when they take input from $Z_{< \ell}$. Hence, the $(\ell+1)$-th epoch transcripts to players $[-1:\ell-1]$ are independent of the choice of $z_{\ell}\in Z_{\ell}$. We can then use a greedy strategy to select transcripts that leak the least amount of information about $S_{\ge \ell+1}$, without consulting to the value of $z_{\ell} \in Z_{\ell}$.

\subsubsection{Step 1: Choosing the set $Z_{\ell}$}
\label{sec:choose_z_l}
%Recall $R_{\geq \ell}$ includes all $(z_{L}, \ldots, z_{\ell}) \in [A_{L}\times \cdots A_{\ell}]$ that are consistent with $\Lambda^{(\ell)}$, and 

Recall the size of $R_{\geq \ell}$ satisfies 
\begin{align}
|R_{\geq \ell}| \geq |A_{L}| \times  \cdots \times |A_{\ell}|/\Delta_{\ell}.
\label{eq:inductive1}
\end{align}
We would like to select a rectangular subset from $R_{\geq \ell}$. The proof of Lemma \ref{lem:joint-set} is deferred to Section \ref{sec:tech}.
\begin{lemma}
\label{lem:joint-set}
There exists a subset $S^{(\ell)} \subseteq R_{\geq \ell}$ such that 
\begin{itemize}
\item $S^{(\ell)} = S^{(\ell)}_1 \times S^{(\ell)}_2$, where $S^{(\ell)}_1 \subseteq A_{L} \times \cdots \times A_{\ell+1}$, $S^{(\ell)}_2 \subseteq A_{\ell}$, with size 
\[
|S^{(\ell)}_1| \geq |A_{L}|\cdots |A_{\ell+1}|/\Delta_{\ell}^{2x_\ell} \quad \text{and}\quad |S^{(\ell)}_2| = x_{\ell}.
\]
\item We have
\begin{align*}
|\{i_{\ell}: i_{\ell} = \wt{z}_{\ell}(\wt{w}_{\ell-1}, \wt{i}_{\ell-1}) \text{ for some }\wt{w}_{\ell-1} \in [n_{\ell-1}],  \wt{i}_{\ell-1} \in \mI_{\ell-1}, \wt{z}_{\ell} \in S_2^{(\ell)}\}| \geq \Theta_{\ell}
\end{align*}
\end{itemize}
\end{lemma}

With Lemma \ref{lem:joint-set} in hand, we take $Z_{\ell} = S^{(\ell)}_2$ and $S_{\geq \ell+1} = S_1^{(\ell)}$.
%\subseteq [A_{\ell}]$ and we have proved $|Z_{\ell}| = x_{\ell}$. 

Next, we are going to fix the transcript $\Lambda^{(\ell+1)}$. Recall we need to fix all transcripts from players $[\ell+1: L]$ to players $[-1:\ell]$ in the first $\ell+1$ epochs, when players $[-1:\ell]$ receive input from $Z_{\leq \ell} = Z_{-1} \times Z_0 \times \cdots \times Z_{\ell}$.
We proceed in a few steps.

\subsubsection{Step 2.1: Fixing the transcript to players $[-1:\ell-1]$ in the first $\ell$ epochs}
\label{sec:inductive-step1}
First, we fix the transcript from players $j \in [\ell+1: L]$ to players $i \in [-1:\ell-1]$ in the first $\ell$ epochs. We simply use $\Lambda^{(\ell)}$, that is, for any $\wt{z}_{\ell}\in Z_{\ell}, \ldots, \wt{z}_{i} \in Z_i$,
%\begin{align*}
%\Lambda_{j, i}^{(\ell+1, \ell')} = (\Lambda_{j, i}^{(\ell, \ell')})^{x_{\ell}}  \subseteq \mathsf{domain}(\Pi_{j, i}^{(\ell')}(\times_{t\in [\ell: i]}Z_t))
%\end{align*}
\begin{align}
\Lambda_{j, i}^{(\ell+1, \ell')}(\wt{z}_{\ell}, \ldots, \wt{z}_{i}) = \Lambda_{j, i}^{(\ell, \ell')}(\wt{z}_{\ell-1}, \ldots, \wt{z}_{i}).  \quad \forall j \in [\ell+1: L], i\in [-1: \ell-1], \ell' \in [\ell]  \label{eq:fix1}
\end{align}
%In words, we assume the transcript from party $j$ to party $i$ at layer $\ell'$, under input $\wt{z}_{\ell}, \ldots, \wt{z}_{i}$, equals to $\Lambda_{j, i}^{(\ell, \ell')}(\wt{z}_{\ell-1}, \ldots, \wt{z}_{i})$, independent of the value of $\wt{z}_{\ell}$.

We claim that $S_{\geq \ell+1} \subseteq A_{L}\times \cdots \times A_{\ell+1}$ is consistent with $\Lambda^{(\ell+1)}$ up to this point. 
Formally, we have
\begin{lemma}
\label{lem:step1}
The set $S_{\geq \ell+1}$ is consistent with $\{\Lambda^{(\ell, \ell')}_{j, i}\}_{j\in [\ell+1:L], i \in [-1:\ell-1], \ell'\in [\ell]}$. Formally, for any $\wt{z}_{\geq \ell+1} \in S_{\geq \ell+1}$ and $\wt{z}_{< \ell+1} \in Z_{< \ell+1}$, one has
\begin{align*}
\Pi_{j, i}^{(\ell')}(\wt{z}_L, \ldots, \wt{z}_i) = \Lambda_{j, i}^{(\ell+1, \ell')}(\wt{z}_\ell, \ldots, \wt{z}_i). 
\end{align*}
for any $j\in [\ell+1:L], i \in [-1:\ell-1], \ell'\in [\ell]$.
\end{lemma}
\begin{proof}
By Lemma \ref{lem:joint-set} and our choice of $Z_{\ell}$, we have $(\wt{z}_{L}, \ldots, \wt{z}_{\ell}) \in S_{\geq \ell+1}\times Z_{\ell} \subseteq R_{\geq \ell}$, and therefore
\begin{align*}
\Pi_{j, i}^{(\ell')}(\wt{z}_L, \ldots, \wt{z}_i) = \Lambda_{j, i}^{(\ell, \ell')}(\wt{z}_{\ell-1}, \ldots, \wt{z}_i) = \Lambda_{j, i}^{(\ell+1, \ell')}(\wt{z}_{\ell}, \wt{z}_{\ell-1}, \ldots, \wt{z}_i).
\end{align*}
where the first step follows from the definition of $R_{\geq \ell}$, the second step follows from Eq.~\eqref{eq:fix1}.
%By our choice of $\Lambda_{j, i}^{(\ell+1, \ell')}$ (see Eq.~\eqref{eq:fix1}), we have
%\begin{align*}
%\Pi_{j, i}^{(\ell')}(\wt{z}_L, \ldots, \wt{z}_i) = \Lambda_{j, i}^{(\ell+1, \ell')}(\wt{z}_{\ell}, \wt{z}_{\ell-1}, \ldots, %\wt{z}_i).
%\end{align*}
%This completes the proof.
\end{proof}

\subsubsection{Step 2.2: Fixing the transcript to players $[-1:\ell-1]$ at the $(\ell+1)$-th epoch}  
\label{sec:inductive-step2}
Next, we fix the transcript from players $j \in [\ell+1: L]$ to players $i \in [-1:\ell-1]$ at the $(\ell+1)$-th epoch. Our key insight is that $Z_{\ell}$ is indistinguishable to players $[-1:\ell-1]$ when they take input from $Z_{\leq \ell-1}$, hence their transcripts are independent of the choice of $z_{\ell}\in Z_{\ell}$.
To this end, consider
\begin{align*}
\Phi^{(\ell+1)} =  \left(\Phi_{j, i}^{(\ell+1)}\right)_{j \in [\ell+1:L], i\in [-1:\ell-1]}
\end{align*}
where 
\begin{align*}
\Phi_{j, i}^{(\ell+1)} = \left(\Phi_{j, i}^{(\ell+1)}(\wt{z}_{\ell-1}, \ldots, \wt{z}_i)\right)_{\wt{z}_{\ell-1} \in Z_{\ell} \ldots, \wt{z}_i \in Z_{i}} \quad \text{and} \quad \Phi_{j, i}^{(\ell+1)}(\wt{z}_{\ell-1}, \ldots, \wt{z}_i) \in \mathsf{domain}(\Pi_{j, i}^{(\ell+1)})
\end{align*}
Comparing $\Lambda_{j,i}^{(\ell+1, \ell+1)}$ with $\Phi_{j, i}^{(\ell+1)}$, we note that $\Phi_{j, i}^{(\ell+1)}$ does not have dependence on $\wt{z}_{\ell} \in Z_{\ell}$.  
%\begin{align*}
%\Gamma = (\Gamma_{j, i}(\wt{z}_{\ell-1}, \ldots, \wt{z}_i))_{j \in [\ell+1:L], i\in \{e\}\cup [0:\ell-1], (\wt{z}_{\ell-1}, \ldots, \wt{z}_i) \in Z_{\ell-1} \times \cdots \times Z_{i}}
%\end{align*}
%where
%\begin{align*}
%\Gamma_{j, i}({\wt{z}_{\ell-1}, \ldots, \wt{z}_{i}}) \in \{0, 1\}^{N_{i-1} \cdot  2Hdp} \quad \text{when} \quad  i \in [1:L-1]
%\end{align*}
%and 
%\begin{align*}
%\Gamma_{j, i}({\wt{z}_{\ell-1}, \ldots, \wt{z}_{i}})
%\in \{0,1\}^{2Hdp} \quad \text{when} \quad i = 0, \e.
%\end{align*}
%\begin{align*}
%\Lambda_{j, i}^{(\ell+1, \ell+1)}(\wt{z}_{\ell}, \ldots, \wt{z}_{i}) = \Gamma_{\wt{z}_{\ell-1}, \ldots, \wt{z}_{i}}
%\end{align*}

For any $\Phi^{(\ell+1)}$, define 
\begin{align}
S(\Phi^{(\ell+1)}) := \left\{
\begin{matrix}
(\wt{z}_{L}, \ldots, \wt{z}_{\ell+1}) \in S_{\geq \ell+1}: \\
\text{such that }\Pi_{j, i}^{(\ell+1)}(\wt{z}_{L}, \ldots, \wt{z}_{i}) = \Phi_{j, i}^{(\ell+1)}({\wt{z}_{\ell-1}, \ldots, \wt{z}_{i}})\\
\forall \wt{z}_{\ell} \in Z_\ell, \ldots, \wt{z}_{i} \in Z_i, j\in [\ell+1:L], i\in [-1:\ell-1]
\end{matrix} 
\right\} \label{eq:s-phi}
\end{align}
In words, $S(\Phi^{(\ell+1)})$ include all $(\wt{z}_{L}, \ldots, \wt{z}_{\ell+1}) \in S_{\geq \ell+1}$ that are consistent with the transcript $\Phi^{(\ell + 1)}$.
Our key observation is  
\begin{lemma}
\label{lem:key-observation}
We have
\[
\bigcup_{\Phi^{(\ell+1)}} S(\Phi^{(\ell+1)}) = S_{\geq \ell+1}. 
\]
\end{lemma}
\begin{proof}
It suffices to prove that, for any $(\wt{z}_{L}, \ldots, \wt{z}_{\ell+1}) \in S_{\geq \ell+1}$, $\wt{z}_{\ell-1}\in Z_{\ell-1}, \ldots, \wt{z}_{i} \in Z_{i}$, $j \in [\ell+1:L], i\in  [-1: \ell-1]$, the transcript $\Pi_{j, i}^{(\ell+1)}(\wt{z}_{L}, \ldots,\wt{z}_{\ell+1}, z_{\ell},\wt{z}_{\ell-1},\ldots  \wt{z}_{i})$ is the same for every $z_{\ell} \in Z_{\ell}$.

To prove this, note that the transcript is determined by the information state $X^{(\ell)}_{j}$ and $X^{(\ell)}_{i}$. 
It is clear that $X^{(\ell)}_{j}$ does not change with the choice of $z_{\ell} \in Z_{\ell}$ since $j > \ell$. It remains to prove that $X^{(\ell)}_{i}$ also does not change with $z_{\ell} \in Z_{\ell}$. To this end, we prove that all information states 
$\{X_{r}^{(\ell')}\}_{r \in [i:\ell-1], \ell'\in [\ell]}$ does not change with $z_{\ell}$.

Recall we have fixed the value of $(\wt{z}_{L}, \ldots, \wt{z}_{\ell+1}) \in S_{\geq \ell+1}$ and $\wt{z}_{\ell-1} \in Z_{\ell-1}, \ldots, \wt{z}_{i} \in Z_{i}$. We prove by induction on $r = \ell-1, \ldots, i$. 
When $r = \ell-1$, the information state $X_{\ell-1}^{(\ell')}$ is determined by $\wt{z}_{\ell-1}$ and $\Pi_{t, \ell-1}^{(\ell'')}(\wt{z}_{L}, \ldots, \wt{z}_{\ell+1},z_{\ell}, \wt{z}_{\ell-1} )$ ($t \in [\ell: L], \ell'' \in [\ell']$), since $(\wt{z}_L, \ldots, \wt{z}_{\ell+1},z_{\ell}) \in R_{\geq \ell}$ for every $z_{\ell}\in Z_{\ell}$, we have that $\Pi_{t, \ell-1}^{(\ell'')}(\wt{z}_{L}, \ldots, \wt{z}_{\ell+1},z_{\ell}, \wt{z}_{\ell-1} ) = \Lambda_{t, \ell-1}^{(\ell, \ell'')}(\wt{z}_{\ell-1})$, which is the same for every $z_{\ell}\in Z_{\ell}$. This finishes the proof of the base case. 

Now suppose the induction continues to hold for $r + 1$, for $r$, the information state $X_{r}^{(\ell')}$ is determined by $\wt{z}_{r}$, $\Pi_{t, r}^{(\ell'')}(\wt{z}_{L}, \ldots, \wt{z}_{\ell+1}, z_{\ell}, \wt{z}_{\ell-1}, \ldots, \wt{z}_{r})$ ($t \in [r+1: L], \ell''\in [\ell']$). 
For $t \in [\ell: L]$, since $(\wt{z}_L, \ldots, \wt{z}_{\ell+1}, z_{\ell}) \in R_{\geq \ell}$, we have that $\Pi_{t, r}^{(\ell'')}(\wt{z}_{L}, \ldots, \wt{z}_{\ell+1}, z_{\ell}, \wt{z}_{\ell-1}, \ldots, \wt{z}_{r}) = \Lambda_{t, r}^{(\ell, \ell'')}(\wt{z}_{\ell-1}, \ldots, \wt{z}_{r})$, which are the same for every $z_{\ell}\in Z_{\ell}$. For $t \in [r+1: \ell-1]$, we have proved $X_{t}^{(\ell'')}$ are the same for every $z_{\ell}\in Z_{\ell}$, so does $\Pi_{t, r}^{(\ell'')}(\wt{z}_{L}, \ldots, \wt{z}_{\ell+1}, z_{\ell}, \wt{z}_{\ell-1}, \ldots, \wt{z}_{r})$. This completes the induction and finish the proof.
%To see this, recall $(\wt{z}_{L}, \ldots, \wt{z}_{\ell+1}, \wt{z}_{\ell})\in S_1^{(\ell)} \times S_2^{(\ell)}\in R_{\geq \ell}$ for every $\wt{z}_{\ell} \in Z_{\ell}$. This indicates that, for every $\wt{z}_{\ell} \in Z_{\ell}$, the transcript from parties $j \in [\ell: L]$ to party $t \in [i: \ell-1]$ at layer $\ell'\in [\ell]$ equals $\Lambda_{i, j}^{(\ell, \ell')}(\wt{z}_{\ell-1}, \ldots, \wt{z}_{i})$, whose value is independent of $\wt{z}_{\ell}$. Hence, the information state $X_i^{(\ell)}$ of party $i$ after layer $\ell$ is independent of the value of $\wt{z}_{\ell} \in Z_{\ell}$, hence the returned message $\Pi_{j, i}^{(\ell+1)}(\wt{z}_{L}, \ldots, \wt{z}_{i})$ from party $j$ should also be the same for every $\wt{z}_{\ell} \in Z_{\ell}$. This completes the proof. 
\end{proof}

Now we can use a greedy selection strategy. We upper bound the total number of $\Phi^{(\ell+1)}$ and show that there exists at least one $\wt{\Phi}^{(\ell+1)}$ such that $S(\wt{\Phi}^{(\ell+1)})$ has large size. 
The proof of the following Lemma can be found at Section \ref{sec:tech}.

\begin{lemma}
\label{lem:size1}
There exists $\wt{\Phi}^{(\ell+1)}$ such that 
\begin{align*}
|S(\wt{\Phi}^{(\ell+1)})| \geq |A_{L}| \cdots  |A_{\ell+1}| \cdot  2^{-2\sqrt{K} \cdot (x_0\cdots x_{\ell-1})\cdot (n_1\cdots n_{L-1})}
\end{align*}
\end{lemma}
Given Lemma \ref{lem:size1}, we fix the transcript from players $j \in [\ell+1: L]$ to players $i \in  [-1:\ell-1]$ at the $(\ell+1)$-th epoch using $\wt{\Phi}^{(\ell+1)}$, i.e.,
\begin{align}
&~\Lambda_{j, i}^{(\ell+1, \ell+1)}(\wt{z}_{\ell}, \ldots, \wt{z}_i) =  \wt{\Phi}^{(\ell+1)}_{j, i}(\wt{z}_{\ell-1}, \ldots, \wt{z}_{i})\notag\\
\text{for}\quad &~  \forall j \in [\ell+1: L], i \in [-1:\ell-1], \wt{z}_{\ell} \in Z_{\ell}, \ldots, \wt{z}_{i} \in Z_{i}.\label{eq:fix2}
\end{align}
We set $T_{\geq \ell+1} = S(\wt{\Phi}^{(\ell+1)}) \subseteq S_{\geq \ell+1}$. We have
\begin{lemma}
\label{lem:step2}
The set $T_{\geq \ell+1}$ is consistent with $(\Lambda_{j, i}^{(\ell+1, \ell')})_{i \in [-1:\ell-1], j\in [\ell+1: L],\ell'\in [\ell+1]}$. Formally, for any $(\wt{z}_{L}, \ldots, \wt{z}_{\ell+1}) \in T_{\geq \ell+1}$ and $\wt{z}_{<\ell+1} \in Z_{<\ell+1}$, one has
\begin{align}
\Pi_{j, i}^{(\ell')}(\wt{z}_L, \ldots, \wt{z}_i) = \Lambda_{j, i}^{(\ell+1, \ell')}(\wt{z}_\ell, \ldots, \wt{z}_i). \label{eq:step2-1}
\end{align}
for any $j\in [\ell+1:L], i \in [-1:\ell-1], \ell'\in [\ell+1]$.
Moreover, we have
\begin{align}
|T_{\geq \ell+1}| \geq |A_{L}\times \cdots \times A_{\ell+1}| \cdot  2^{-2\sqrt{K} \cdot (x_0\cdots x_{\ell-1})\cdot (n_1\cdots n_{L-1})}.
\label{eq:step2-2}
\end{align}
\end{lemma}
\begin{proof}
Since $T_{\geq \ell+1} \subseteq S_{\geq \ell+1}$, by Lemma \ref{lem:step1}, Eq.~\eqref{eq:step2-1} holds for any $\ell'\in [\ell]$.
By Eq.~\eqref{eq:s-phi} and Eq.~\eqref{eq:fix2}, we also know that 
\begin{align*}
\Pi_{j, i}^{(\ell+1)}(\wt{z}_L, \ldots, \wt{z}_i) = \wt{\Phi}^{(\ell+1)}_{j, i}(\wt{z}_{\ell-1}, \ldots, \wt{z}_{i}) = \Lambda_{j, i}^{(\ell+1, \ell+1)}(\wt{z}_{\ell}, \ldots, \wt{z}_i)
\end{align*}
This completes the proof of Eq.~\eqref{eq:step2-1}. Eq.~\eqref{eq:step2-2} follows directly from Lemma \ref{lem:size1}.
\end{proof}

\subsubsection{Step 2.3: Fixing the transcript to player $\ell$} 
\label{sec:inductive-step3}
Finally, we fix the transcript from the player $j \in [\ell+1: L]$ to the player $\ell$ at the first $(\ell+1)$ epochs. This follows from the a greedy selection strategy.
%This follows from a simple averaging argument.
Let 
\begin{align*}
\Psi = \left(\Psi_{j, \ell}^{(\ell')}(\wt{z}_{\ell})\right)_{j\in [\ell+1:L], \ell' \in [\ell+1], \wt{z}_{\ell} \in Z_{\ell}}
\quad
\text{where}
\quad 
\Psi_{j, \ell}^{(\ell')}(\wt{z}_{\ell}) \in \mathsf{domain}(\Pi_{j, \ell}^{(\ell')})
%\{0, 1\}^{2Hdp \cdot N_{\ell-1}}.
\end{align*}
Define 
\begin{align}
T(\Psi):= \left\{
\begin{matrix}
(\wt{z}_{L}, \ldots, \wt{z}_{\ell+1}) \in T_{\geq \ell+1}: \\
\Pi_{j, \ell}^{(\ell')}(\wt{z}_{L}, \ldots, \wt{z}_{\ell}) = \Psi_{j, \ell}^{(\ell')}(\wt{z}_{\ell}) \quad \forall \wt{z}_{\ell} \in Z_\ell, \ell'\in [\ell+1], j \in [\ell+1:L]
\end{matrix} 
\right\} \label{eq:t-psi}
\end{align}

%Define $T(\Phi) \subseteq T^{(\ell)}$ such that for any (\wt{z}_{L}, \ldots \wt{z}_{\ell+1}) \in T(\Phi)$, it satisfies that, for every $\wt{z}_{L}\in Z_{\ell}$:
%\begin{align*}
%\Pi_{j, \ell}^{(\ell')}(\wt{z}_{L}, \ldots, \wt{z}_{\ell}) = \Phi_{j, \ell}^{(\ell')}(\wt{z}_{\ell}) \quad \forall j \in [\ell+1:L], \ell'\in [\ell+1]
%\end{align*}

We can upper bound the number of different $\Psi$, and use an average argument to obtain the following Lemma. Its proof can be found at Section \ref{sec:tech}.
\begin{lemma}
\label{lem:size2}
There exists $\wt{\Psi}$ such that $|T(\wt{\Psi})| \geq |A_{L}|  \cdots |A_{\ell+1}|/\Delta_{\ell+1}$.
\end{lemma}

Given Lemma \ref{lem:size2}, we fix the transcripts from players $j \in [\ell+1: L]$ to players $\ell$ at the first $(\ell+1)$-th epochs using $\wt{\Psi}$. In particular, we take
\begin{align}
\Lambda_{j, \ell}^{(\ell+1, \ell')}(\wt{z}_{\ell}) = \wt{\Psi}_{j, \ell}^{(\ell')}(\wt{z}_{\ell})  \quad  \forall j\in [\ell+1:L], \ell'\in [\ell+1], \wt{z}_{i}\in Z_{i}. \label{eq:fix3}
\end{align}

We take $R_{\geq\ell+1} = T(\wt{\Psi})$ and we have

\begin{lemma}
\label{lem:step3}
$R_{\geq\ell+1}$ is consistent with $\Lambda^{(\ell+1)}$. Formally, for any $\wt{z}_{\geq\ell+1} \in R_{\geq\ell+1}$ and $\wt{z}_{<\ell+1}\in Z_{<\ell+1}$, one has
\begin{align}
\Pi_{j, i}^{(\ell')}(\wt{z}_L, \ldots, \wt{z}_i) = \Lambda_{j, i}^{(\ell+1, \ell')}(\wt{z}_\ell, \ldots, \wt{z}_i).  \label{eq:step3-1}
\end{align}
holds for any $j\in [\ell+1:L], i \in [-1:\ell], \ell'\in [\ell+1]$.
Moreover, we have
\begin{align}
|R_{\geq\ell+1}| \geq |A_{L}| \cdots  |A_{\ell+1}| /\Delta_{\ell+1}\label{eq:step3-2}
\end{align}
\end{lemma}
\begin{proof}
Since $R_{\geq\ell+1} \subseteq T_{\geq\ell+1}$, by Lemma \ref{lem:step2}, it is consistent with $\{\Lambda_{j, i}^{(\ell+1, \ell')}\}_{j \in [\ell+1: L], i \in [-1:\ell-1], \ell' \in [\ell+1]}$.
Combining Eq.~\eqref{eq:t-psi}\eqref{eq:fix3}, we have
\[
\Pi_{j, \ell}^{(\ell')}(\wt{z}_L, \ldots, \wt{z}_\ell) = \Psi_{j, \ell}^{(\ell')}(\wt{z}_{\ell}) = \Lambda_{j, \ell}^{(\ell+1, \ell')}(\wt{z}_\ell). 
\]
This proves Eq.~\eqref{eq:step3-1}. Eq.~\eqref{eq:step3-2} follows directly from Lemma \ref{lem:size2}.
\end{proof}

\subsubsection{Wrap up the induction}
Combining the above three steps (Eq.~\eqref{eq:fix1}\eqref{eq:fix2}\eqref{eq:fix3}), we have fixed the transcript 
\begin{align*}
\Lambda^{(\ell+1)} = &~  \{\Lambda_{j, i}^{(\ell+1, \ell')}\}_{j \in [\ell+1: L], i \in [-1:\ell], \ell' \in [\ell+1]}.
\end{align*}
Now we can wrap up the induction. We need to verify the two inductive hypothesis, i.e., Eq.~\eqref{eq:inductive-hypothesis1}\eqref{eq:inductive-hypothesis2}. 

For the first inductive hypothesis, by Lemma \ref{lem:step3}, we know that $R_{\geq\ell+1}$ is consistent with $\Lambda^{(\ell+1)}$.

For the second inductive hypothesis (i.e., Eq.~\eqref{eq:inductive-hypothesis1}), we have proved in Lemma \ref{lem:step3} that $|R_{\geq \ell+1}| \geq  |A_{L}|\cdots |A_{\ell+1}|/\Delta_{\ell+1}$. 
%This proves Eq.~\eqref{eq:inductive-hypothesis1}.

For the last inductive hypothesis (i.e., Eq.~\eqref{eq:inductive-hypothesis2}), recall we take $Z_{\ell} = S_{2}^{(\ell)}$ in Lemma \ref{lem:joint-set}, we have that 
\begin{align*}
|\mI_{\ell}| = |\{i_{\ell}: i_{\ell} = \wt{z}_{\ell}(\wt{w}_{\ell-1}, \wt{i}_{\ell-1}) \text{ for some }\wt{w}_{\ell-1} \in [n_{\ell-1}],  \wt{i}_{\ell-1} \in \mI_{\ell-1}, \wt{z}_{\ell} \in S_2^{(\ell)}\}| \geq \Theta_{\ell}
\end{align*}
This proves Eq.~\eqref{eq:inductive-hypothesis2}) and completes the induction step.

\subsection{Proof of remaining technical lemmas}
\label{sec:tech}

Finally, we provide the missing proofs for the technical lemmas used before.

We first prove Lemma \ref{lem:initial}.
\begin{proof}[Proof of Lemma \ref{lem:initial}]
Consider the following greedy approach. Initially, we set $D_0 = \emptyset$. For $\tau=1,2,\ldots$, we add $\zeta_{\tau} \in S$ to $D_{\tau} = D_{\tau-1}\cup \{\zeta_{\tau}\}$ if it satisfies 
\begin{align}
|\{\zeta_{\tau}(i_0): i_0\in Z_0\} \setminus \{ \wt{z}_1(i_0): \wt{z}_1 \in D_{\tau-1}, i_0 \in Z_0 \}| \geq 8^{-L}x_0. \label{eq:add}
\end{align}
If the process continues to $\tau = x_{1} + 1$, then we take $Z_1 = D_{x_1}$, we have $|Z_1| = |D_{x_1}| = x_1$ and by Eq.~\eqref{eq:add}, it satisfies the requirement of Eq.~\eqref{eq:initial-requirement}. If the process stops for $\tau \leq x_1$ and $|\{\wt{z}_1(i_0): \wt{z}_1 \in D_{\tau-1}, i_0 \in Z_0 \}| \geq 8^{-L}x_0x_1$, then we can take $Z_{1}$ to be the union of $D_{\tau-1}$ and arbitrary $x_1 - \tau + 1$ elements from $S$, we have $|Z_1| =x_1$ and it satisfies the requirement of Eq.~\eqref{eq:initial-requirement}.

Otherwise, suppose it stops for $\tau \leq x_1$ and  $|\{\wt{z}_1(i_0): \wt{z}_1 \in D_{\tau-1}, i_0 \in Z_0 \}| < 8^{-L}x_0x_1$.
We claim that for every $\xi \in S$, its entries $\xi(1), \ldots, \xi(x_0)$ take value from either $\{\wt{z}_1(i_0): \wt{z}_1 \in D_{\tau-1}, i_0 \in Z_0 \}$, or a subset from $[N_0]\setminus \{\wt{z}_1(i_0): \wt{z}_1 \in D_{\tau-1}, i_0 \in Z_0 \}$ of size at most $8^{-L}x_0$, the total number of such $\zeta$ is at most 
\begin{align*}
\binom{N_0}{8^{-L}x_0} \cdot (8^{-L}x_0x_1)^{x_0} \cdot (N_0)^{N_0 - x_0} \leq |A_0| \cdot (m^{8^{-L}} \cdot 8^{-L}x_0x_1/m)^{x_0} < |A_0| \cdot 2^{-2Hdp \cdot (n_1\cdots n_{L-1})}
\end{align*}
%\hongxun{Is $8^{-L} x_0 x_1$ the upper bound on $|\{\wt{z}_1(i_0): \wt{z}_1 \in D_{\tau-1}, i_0 \in Z_0 \}|$? This is not necessarily true if we only stop for $\tau < x_1$. We would also have to stop when $|\{\wt{z}_1(i_0): \wt{z}_1 \in D_{\tau-1}, i_0 \in Z_0 \}| \geq 8^{-L} x_1 x_1$ and declare we have already success. }
The first step follows from $A_0 = N_0^{N_0}, N_0 = m$, the second step follows from (1) $m^{8^{-L}}\cdot 8^{-L} x_0x_1 \leq m/2$ and $x_{0} > 2Hdp (n_1\cdots n_{L-1})$ (see Eq.~\eqref{eq:parameter1}\eqref{eq:parameter3} for the choice of parameters).
This contradicts with Eq.~\eqref{eq:initial1} and completes the proof.
\end{proof}

The proof of Lemma \ref{lem:joint-set} involves several intermediate steps, and we present it here.
\begin{proof}[Proof of Lemma \ref{lem:joint-set}]
We simplify the notation a bit in the proof.
We write
\begin{align}
R := R_{\geq \ell},\quad \mA := A_{L} \times \cdots\times A_{\ell+1}, \quad \mB := A_{\ell}, \quad \Delta := \Delta_{\ell}, \quad x := x_{\ell}. \label{eq:simple-notation}
\end{align}
For any $b \in \mB$, let
\[
R_b = \{a\in \mA: (a, b)\in R\} \subseteq \mA 
\]
By Eq.~\eqref{eq:inductive-hypothesis1}\eqref{eq:simple-notation}, we have 
\begin{align}
\sum_{b\in \mB}|R_b| = |R| \geq |\mA||\mB|/\Delta \label{eq:joint0}
\end{align}
For any $b \in \mB = [N_{\ell-1}]^{N_{\ell-1}}$, define 
\[
I(b) := \{b(\wt{w}_{\ell-1}, \wt{i}_{\ell-1}): \wt{w}_{\ell-1} \in [n_{\ell-1}], \wt{i}_{\ell-1} \in \mI_{\ell-1} \} \subseteq [N_{\ell-1}]
\]

The key step is to prove
\begin{lemma}\label{lem:sequence}
There exists two sequences $b_{1}, \ldots, b_{x} \in \mB$ and $\mA_1, \ldots, \mA_{x} \subseteq \mA$ that satisfy
\begin{itemize}
\item $\mA_{x} \subseteq \mA_{x-1} \cdots \subseteq \mA_{1} \subseteq \mA_0 = \mA$, and $|\mA_{t}| \geq |\mA_{t-1}|/\Delta^{2}$ ($t \in [x]$)
\item $\mA_{t} \subseteq R_{b_t}$ ($t \in [x]$)
\item $|I(b_t) \setminus \cup_{t'\in [t-1]} I(b_{t'})| \geq 8^{-L} n_{\ell-1} |\mI_{\ell-1}|$ ($t \in [x]$)
\item $\sum_{b \in \mB} |R_b \cap A_t| \geq (1 - \frac{1}{x})^{t}\frac{|\mA_t||\mB|}{\Delta}$ ($t \in [0: x]$)
\end{itemize}
\end{lemma}
With Lemma \ref{lem:sequence}, we can easily finish the proof of Lemma \ref{lem:joint-set}. To see this, we take 
\[
S_{1}^{(\ell)} = \mA_{x}  \quad \text{and} \quad  S_{2}^{(\ell)} = \{b_1, \ldots, b_{x}\}.
\]
Then we have $|S_{1}^{(\ell)}| \geq |\mA|/\Delta^{2x}$ due to the first claim in Lemma \ref{lem:sequence} and $|S_{2}^{(\ell)}| = x$. Moreover, by the second claim in Lemma \ref{lem:sequence}, we have $S_{1}^{(\ell)} = \mA_x \subseteq R_b$ for any $b \in S_{2}^{(\ell)}$, hence we have $S^{{\ell}} = S_1^{(\ell)} \times S_2^{(\ell)} \subseteq R$. 

For the second property of Lemma \ref{lem:joint-set}, we have
\begin{align*}
&~ |\{i_{\ell}: i_{\ell} = \wt{z}_{\ell}(\wt{w}_{\ell-1}, \wt{i}_{\ell-1}) \text{ for some }\wt{w}_{\ell-1} \in [n_{\ell-1}],  \wt{i}_{\ell-1} \in \mI_{\ell-1}, \wt{z}_{\ell} \in S_2^{(\ell)}\}| \\
=  &~|\cup_{t\in [x]}I(b_x)| \geq 8^{-L} x_{\ell} n_{\ell-1}|\mI_{\ell-1}| \geq 8^{-L} x_{\ell} n_{\ell-1}\Theta_{\ell-1} \geq \Theta_{\ell}
\end{align*}
The second step follows from the third claim of Lemma \ref{lem:sequence}, the third step follows from the inductive hypothesis of Lemma \ref{lem:main} (Eq.~\eqref{eq:inductive-hypothesis2}), the last step follows from the definition of $\Theta_{\ell}$ (see Eq.~\eqref{eq:def-delta}).
This completes the proof of Lemma \ref{lem:joint-set}.
\end{proof}

We next provide the proof of the key Lemma.
\begin{proof}[Proof of Lemma \ref{lem:sequence}]
We prove by induction. The case of $t = 0$ follows directly from Eq.~\eqref{eq:joint0}. 
Suppose the claim holds up to $t$, then for $t + 1$, let 
\begin{align*}
R_b' := R_{b} \cap \mA_t \quad \quad \forall b\in \mB.
\end{align*}
By the fourth induction hypothesis, we have
\begin{align}
\sum_{b \in \mB}|R_b'| = \sum_{b \in \mB} |R_{b} \cap \mA_t| \geq \big(1-\frac{1}{x}\big)^t \cdot \frac{|\mA_{t}||\mB|}{\Delta}.\label{eq:joint1}
\end{align}
Consider the set 
\[
J_{t} := \cup_{t'\in [t]}I(b_{t'})
\]
its size satisfies
\begin{align*}
8^{-L} t n_{\ell-1} |\mI_{\ell-1}| \leq |J_{t}| \leq t n_{\ell-1} |\mI_{\ell-1}|
\end{align*}
where the LHS follows from the third inductive hypothesis, i.e.,
$
|J_{t}| = \sum_{t'\in [t]}|J_{t'} \setminus J_{t'-1}| \geq t\cdot 8^{-L}  n_{\ell-1} |\mI_{\ell-1}|
$,
the RHS holds since $|J_{t}| \leq \sum_{t'\in [t]}|I(b_{t'})| \leq t \cdot n_{\ell-1} |\mI_{\ell-1}|$.

Define 
\begin{align}
\label{eq:def-c}
C: =\{b \in \mB: |I(b) \setminus J_t| < 8^{-L}n_{\ell-1}|\mI_{\ell-1}|\} \subseteq \mB
\end{align}
We have the following claim, whose proof is deferred till the end.
\begin{lemma}
\label{lem:size-c}
We have  $|C|  \leq \frac{|\mB|}{\Delta^2}$.
\end{lemma}

Define
\begin{align}
R_b'' :=\left\{
\begin{matrix}
R_b' &  b \notin C  \\
\emptyset & b\in C
\end{matrix}
\right. \quad \quad \forall b \in \mB. \label{eq:def-r''}
\end{align}
That is, $R_{b}''$ is set to empty if $I(b)$ does not differ too much from $J_t$.

By Lemma \ref{lem:size-c}, we have 
\begin{align}
\sum_{b\in \mB}|R_{b}''| = &~ \sum_{b\in \mB}|R_{b}'| - \sum_{b\in C}|R_{b}'| \geq \left(1-\frac{1}{x}\right)^t \cdot \frac{|\mA_{t}||\mB|}{\Delta} -  \frac{|\mB|}{\Delta^2}\cdot |\mA_t| \label{eq:joint2}.
\end{align}
The first step follows from the definition of $R_b''$ (see Eq.~\eqref{eq:def-r''}), the second step follows from  Eq.~\eqref{eq:joint1}, Lemma \ref{lem:size-c} and $|R_b'| \leq |\mA_t|$.

Now we are going to select $b_{t+1} \in \mB$ and $\mA_{t+1}\subseteq \mA_t$. To this end, consider the following greedy process. Initially, set $D = \emptyset$ and $\tau = 0$. If there exists $\beta_{\tau} \in \mB$ such that  
\begin{align}
R_{\beta_{\tau}}''\setminus \left(\bigcup_{\beta\in D} R_{\beta}''\right) \geq \frac{|\mA_{t}|}{\Delta^2},  \label{eq:joint3}
\end{align}
then define
\begin{align}
O_{\tau} := R_{\beta_{\tau}}''\setminus \left(\bigcup_{\beta\in D} R_{\beta}''\right) \subseteq \mA_t \label{eq:joint4}
\end{align}
and update $D \leftarrow D \cup \{\beta_{\tau}\}$ and $\tau = \tau + 1$. 
It is clearly that $D$ is not empty (see Eq~\eqref{eq:joint2}) and the process would stop. Let
\begin{align*}
O := \cup_{\tau \in [0: |D|-1]}O_{\tau} \quad \text{and} \quad
R_b''' := R_b'' \cap O \quad \forall b\in \mB,
\end{align*}
we have 
\begin{align*}
\sum_{b \in \mB} |R_{b}'''| = \sum_{b \in \mB} |R_{b}''\cap O| =   \sum_{b \in \mB} |R_{b}''| -  \sum_{b \in \mB} |R_{b}''\setminus O| \geq &~ \left(1-\frac{1}{x}\right)^t \cdot \frac{|\mA_{t}||\mB|}{\Delta} -  \frac{|\mB|}{\Delta^2}\cdot |\mA_t| - |\mB| \cdot \frac{|\mA_t|}{\Delta^2}\\
\geq &~ \left(1-\frac{1}{x}\right)^{t+1} \cdot \frac{|\mA_{t}||\mB|}{\Delta}
\end{align*}
Here the third step follows from Eq.~\eqref{eq:joint2} and the selection strategy (i.e. Eq.~\eqref{eq:joint3}), the last step holds since $\Delta \gg x$.

Since $\{O_{\tau}\}_{\tau \in [0:|D|-1]}$ are disjoint, we claim that, there exists $\tau \in [0: |D|-1]$, such that 
\begin{align}
\sum_{b \in \mB} |R_{b}''' \cap O_{\tau}| \geq \left(1-\frac{1}{x}\right)^{t+1} \cdot \frac{|\mA_t||\mB|}{\Delta} \cdot \frac{|O_{\tau}|}{|O|} \geq \left(1-\frac{1}{x}\right)^{t+1}\frac{|O_\tau||\mB|}{\Delta} \label{eq:joint5}.
\end{align}
Here the second step follows from $|\mA_{t}| \geq |O|$.

We take $\mA_{t+1} = O_{\tau}$, $b_{t+1} = \beta_{\tau}$, and we prove the four inductive hypothesis continue to holds.

First, by Eq.~\eqref{eq:joint3}\eqref{eq:joint4}, we have $\mA_{t+1} = O_{\tau} \subseteq \mA_{t}$, and
\begin{align}
|\mA_{t+1}| = |O_{\tau}|  \geq\frac{|\mA_{t}|}{\Delta^2}  \label{eq:joint-inductive1}
\end{align}

Second, by Eq.~\eqref{eq:joint4}, we have
\begin{align}
\mA_{t+1} = O_{\tau} \subseteq  R_{\beta_{\tau}}'' = R_{b_{t+1}}'' \subseteq R_{b_{t+1}}\label{eq:joint-inductive2}
\end{align}

Third, it is clear that $R_{b_{t+1}}''$ is not empty, this implies that $b_{t+1} \notin C$, and therefore
\begin{align}
|I(b_{t+1}) \setminus J_t| \geq 8^{-L} \cdot n_{\ell-1}|\mI_{\ell-1}|.\label{eq:joint-inductive3}
\end{align}

Finally, by Eq~\eqref{eq:joint5}, we have
\begin{align}
\sum_{b \in \mB} |R_b \cap \mA_{t+1}| = \sum_{b \in \mB} |R_b \cap O_{\tau}| \geq \sum_{b \in \mB} |R_{b}''' \cap O_{\tau}| \geq &~  \left(1-\frac{1}{x}\right)^{t+1}\frac{|O_t||\mB|}{\Delta}\notag \\
= &~ \left(1-\frac{1}{x}\right)^{t+1}\frac{|\mA_{t+1}||\mB|}{\Delta}.\label{eq:joint-inductive4}
\end{align}
Combining Eq.~\eqref{eq:joint-inductive1}\eqref{eq:joint-inductive2}\eqref{eq:joint-inductive3}\eqref{eq:joint-inductive4}, we complete the proof.
\end{proof}

\begin{proof}[Proof of Lemma \ref{lem:size-c}]
We compute the size of $C$. For any $b \in C \subseteq A_{\ell} = [N_{\ell-1}]^{N_{\ell-1}}$, one can view it as a mapping from $[N_{\ell-1}]$ to $[N_{\ell-1}]$.
By Eq.~\eqref{eq:def-c}, $b\in C$ if and only if it does not differ much than $J_t$ on entries $(w_{\ell-1}, i_{\ell-1}) \in [n_{\ell-1}] \times \mI_{\ell-1}$. For entries in $[N_{\ell}]\setminus ([n_{\ell-1}] \times \mI_{\ell-1})$, it poses no constraints so it has
$
(N_{\ell-1})^{N_{\ell-1} - n_{\ell-1}|\mI_{\ell-1}|}
$
number of choices.

For entries in $[n_{\ell-1}] \times \mI_{\ell-1}$, it could take value from $J_t$, or from a subset of at size at most $8^{-L} n_{\ell-1}|\mI_{\ell-1}|$ in $[N_{\ell}]\setminus ([n_{\ell-1}] \times \mI_{\ell-1})$, the total number of choices is at most
\[
\binom{N_{\ell-1}}{8^{-L}n_{\ell-1}|\mI_{\ell-1}|} \cdot (xn_{\ell-1}|\mI_{\ell-1}|)^{n_{\ell-1} |\mI_{\ell-1}|} \leq (N_{\ell-1}^{8^{-L}} \cdot xn_{\ell-1}|\mI_{\ell-1}|)^{n_{\ell-1} |\mI_{\ell-1}|}
\]
In summary, the size of $C$ is at most
\begin{align*}
&~ (N_{\ell-1})^{N_{\ell-1} - n_{\ell-1} |\mI_{\ell-1}|} \cdot (N_{\ell-1}^{8^{-L}} \cdot xn_{\ell-1}|\mI_{\ell-1}|)^{n_{\ell-1}|\mI_{\ell-1}|}\\
= &~ |\mB| \cdot (N_{\ell-1}^{8^{-L}} \cdot  x_{\ell}n_{\ell-1}|\mI_{\ell-1}| / N_{\ell-1})^{n_{\ell-1}|\mI_{\ell-1}|} \\
\leq &~   |\mB| \cdot (1/2)^{n_{\ell-1}|\mI_{\ell-1}|} \\
\leq &~  |\mB| / \Delta^2 \\
\end{align*}
The first step follows from the definition of $x$ and $\mB$. For the second step, 
%it suffices to prove
%\[
%N_{\ell-1}^{8^{-L}} \cdot  x_{\ell} n_{\ell-1}|\mI_{\ell-1}| \leq \frac{1}%{2} N_{\ell-1}
%\]
it follows from
\begin{align*}
N_{\ell-1}^{8^{-L}} \cdot  x_{\ell} n_{\ell-1}|\mI_{\ell-1}| \leq &~ \sqrt{K} \cdot x_{\ell}n_{\ell-1} (x_0\ldots x_{\ell-1}) \cdot (n_1\ldots n_{\ell-2})\\
\leq &~ \frac{1}{2} m \cdot (n_1\cdots n_{\ell-1})  = \frac{1}{2}N_{\ell-1}
\end{align*}
where we use the fact that $N_{\ell-1}^{8^{-L}} \leq \sqrt{K}$, $|\mI_{\ell-1}| \leq (x_0\ldots x_{\ell-1}) \cdot (n_1\ldots n_{\ell-2})$ in the first step, the second step follows from the choice of parameters (see Eq.~\eqref{eq:parameter1}\eqref{eq:parameter3})

For the third step, we have
\begin{align*}
\Delta^2 = \Delta_{\ell}^2 = 2^{8\sqrt{K}(x_0\ldots x_{\ell-2})\cdot (n_1\ldots n_{L-1})} \leq 2^{8^{-L\ell} \cdot n_{\ell-1}(x_{0} \ldots x_{\ell-1})\cdot (n_1\ldots n_{\ell-2})}  \leq 2^{n_{\ell-1}|\mI_{\ell-1}|}
\end{align*}
The second step follows from the definition of $\Delta_{\ell}$ (see Eq.~\eqref{eq:def-delta}), the third step follows from the choice of parameters Eq.~\eqref{eq:parameter1}\eqref{eq:parameter3}. The last step follows from the induction hypothesis of Lemma \ref{lem:main}.
%it suffices to prove $\Delta_{\ell} \lesssim n_{\ell-1} |\mI_{\ell-1}|$, this is equivalent to 
%\[
%(x_0\ldots x_{\ell-2})\cdot (n_1\ldots n_{L-1}) \lesssim n_{\ell-1}(x_{0} \ldots x_{\ell-1})\cdot (n_1\ldots n_{\ell-2})  %
%\]
%\Binghui{To be polish}
\end{proof}

We next prove Lemma \ref{lem:size1}. First, we need

\begin{lemma}
\label{lem:count1}
The total number of possible $\Phi^{(\ell+1)}$ is at most $2^{\sqrt{K} \cdot (x_0\cdots x_{\ell-1})\cdot (n_1\cdots n_{L-1})}$.
\end{lemma}
\begin{proof}
Recall 
\begin{align*}
\Phi^{(\ell+1)} =  \left(\Phi_{j, i}^{(\ell+1)}\right)_{j \in [\ell+1:L], i\in \{e\}\cup [0:\ell-1]}
\end{align*}
and 
\begin{align*}
\Phi_{j, i}^{(\ell+1)} = \left(\Phi_{j, i}^{(\ell+1)}(\wt{z}_{\ell-1}, \ldots, \wt{z}_i)\right)_{\wt{z}_{\ell-1} \in Z_{\ell-1} \ldots, \wt{z}_i \in Z_{i}} \quad \text{and} \quad \Phi_{j, i}^{(\ell+1)}(\wt{z}_{\ell-1}, \ldots, \wt{z}_i) \in \mathsf{domain}(\Pi_{j, i}^{(\ell+1)}).
\end{align*}

For $i \in [\ell-1]$, $\Phi^{(\ell+1)}_{j, i}({\wt{z}_{\ell-1}, \ldots, \wt{z}_{i}})$ takes value from $\{0, 1\}^{2Hdp \cdot N_{i-1}}$, summing over $\wt{z}_{\ell-1} \in Z_{\ell-1}, \ldots \wt{z}_{i}\in Z_i$, $j \in [\ell+1: L]$, the total number of $(\Phi_{j, i}^{(\ell+1)})_{j\in [\ell+1:L]}$ is at most
\begin{align*}
\left(2^{2Hdp \cdot N_{i-1}}\right)^{x_i\cdots x_{\ell-1}\cdot L} \leq  2^{2HdpL \cdot (x_0\cdots x_{\ell-1})\cdot (n_1\cdots n_{L-1})}
\end{align*}
where the inequality follows from the choice of parameters (see Eq.~\eqref{eq:parameter1}\eqref{eq:parameter2}\eqref{eq:parameter3}).

For $i = 0$, $\Phi_{j, 0}^{(\ell+1)}({\wt{z}_{\ell-1}, \ldots, \wt{z}_{0}})$ takes value from $\{0, 1\}^{2Hdp}$, summing over $\wt{z}_{\ell-1} \in Z_{\ell-1}, \ldots \wt{z}_{0}\in Z_0$ and $j \in [\ell+1: L]$, the total number of $(\Phi_{j, 0}^{(\ell+1)})_{j\in [\ell+1:L]}$ is at most
\begin{align*}
\left(2^{2Hdp}\right)^{x_0\cdots x_{\ell-1}\cdot L}
\leq &~ 2^{2HdpL \cdot (x_0\cdots x_{\ell-1})\cdot (n_1\cdots n_{L-1})}.
\end{align*}

For $i = -1$, $\Phi^{(\ell+1)}_{j, -1}({\wt{z}_{\ell-1}, \ldots, \wt{z}_{0}}, \wt{z}_{-1})$ takes value from $\{0, 1\}^{2Hdp}$, summing over $\wt{z}_{\ell-1} \in Z_{\ell-1}, \ldots \wt{z}_{0}\in Z_0, \wt{z}_{-1} \in [Z_{-1}]$ and $j \in [\ell+1: L]$, $(\Phi_{j, -1}^{(\ell+1)})_{j\in [\ell+1:L]}$ is at most
\begin{align*}
\left(2^{2Hdp}\right)^{(x_0\cdots x_{\ell-1})\cdot (n_1\cdots n_{L-1}) \cdot L} =  2^{2HdpL \cdot (x_0\cdots x_{\ell-1})\cdot (n_1\cdots n_{L-1})}
\end{align*}
Taking a product over all these terms, the total number of $\Phi^{(\ell+1)}$ is bounded by
\[
\left(2^{2HdpL \cdot (x_0\cdots x_{\ell-1})\cdot (n_1\cdots n_{L-1})}\right)^{L} \leq 2^{\sqrt{K} \cdot (x_0\cdots x_{\ell-1})\cdot (n_1\cdots n_{L-1}) }
\]
\end{proof}

Now we can finish the proof of Lemma \ref{lem:size1}.
\begin{proof}[Proof of Lemma \ref{lem:size1}]
By Lemma \ref{lem:key-observation} and Lemma \ref{lem:joint-set}, we have that 
\begin{align*}
\sum_{\Phi^{(\ell+1)}}\left|S(\Phi^{(\ell+1)})\right| = \left|S_{\geq \ell+1}\right| \geq \frac{|A_{L}| \cdots |A_{\ell+1}|}{\Delta_{\ell}^{2x_{\ell}}}.
\end{align*}
By Lemma \ref{lem:count1}, there exists $\wt{\Phi}^{(\ell+1)}$, such that 
\begin{align*}
|S(\wt{\Phi}^{(\ell+1)})| \geq &~ \frac{|A_{L}| \cdots |A_{\ell+1}|}{\Delta_{\ell}^{2x_{\ell}}} \cdot 2^{-\sqrt{K} \cdot (x_0\cdots x_{\ell-1})\cdot (n_1\cdots n_{L-1})}\\
\geq &~ |A_{L}| \cdots |A_{\ell+1}| \cdot 2^{-2\sqrt{K} \cdot (x_0\cdots x_{\ell-1})\cdot (n_1\cdots n_{L-1})}
\end{align*}
Here the second follows from the choice of parameters (see Eq.~\eqref{eq:parameter1}\eqref{eq:parameter3}\eqref{eq:def-delta}).
\end{proof}

We next prove Lemma \ref{lem:size2}
\begin{proof}[Proof of Lemma \ref{lem:size2}]
First, we have $\cup_{\Psi}T(\Psi) = T_{\geq\ell+1}$ and by Lemma \ref{lem:step2}, we have
\[
\sum_{\Psi}|T(\Psi)| = |T_{\geq\ell+1}| \geq |A_{L}| \cdots |A_{\ell+1}| \cdot  2^{-2\sqrt{K} \cdot (x_0\cdots x_{\ell-1})\cdot (n_1\cdots n_{L-1})}.
\]
Moreover, the total number of $\Psi$ is at most 
\[
\left(2^{2Hdp \cdot N_{\ell-1}}\right)^{x_{\ell} \cdot L^2} \leq 2^{\sqrt{K} \cdot (x_0\cdots x_{\ell-1})\cdot (n_1\cdots n_{L-1})}.
\]
Here the equality holds due to the choice of parameters (see Eq~\eqref{eq:parameter1}\eqref{eq:parameter2}\eqref{eq:parameter3}).

Hence, there exists $\wt{\Psi}$ such that 
\begin{align*}
|T(\wt{\Psi})| \geq &~ |A_{L}| \cdots |A_{\ell+1}| \cdot  2^{-2\sqrt{K} \cdot (x_0\cdots x_{\ell-1})\cdot (n_1\cdots n_{L-1})} \cdot 2^{- \sqrt{K} \cdot (x_0\cdots x_{\ell-1})\cdot (n_1\cdots n_{L-1})} \\
\geq &~ |A_{L}| \cdots |A_{\ell+1}|/\Delta_{\ell+1}
\end{align*}
The last step follows from the definition of $\Delta_{\ell+1}$ (see Eq.~\eqref{eq:def-delta}).
This completes the proof.
\end{proof}

\newpage
\bibliographystyle{alpha}
\bibliography{ref}

\newpage
\appendix
\section{Missing proof from Section \ref{sec:application}} \label{appendix:missing-proof}
We sketch the high level proof idea of Corollary \ref{cor:depth-size}, Corollary \ref{cor:encoder} and Corollary \ref{cor:cot}. Note that the lower bound parts of these corollaries all follow directly from Theorem~\ref{thm:main}, so in the following it suffices to prove the upper bound parts.

\paragraph{Retrieval head.} First, we observe that an attention head could implement the following retrieval operation. Let $i \in [n]$, suppose $x_{j}^{(\ell)}$ (i.e., the input to the $\ell$-th layer at position $j$) contains $a_j \in \{0,1\}^D$ and $b_{j} \in \{0,1\}^{D}$ for any previous position $j \leq i-1$ and suppose $x_{i}^{(\ell)}$ contains a ``query'' $a$. 
The retrieval task at position $i$ (of layer $\ell$) is to find the position $j$ such that $a_j = a$ and return the value of $b_j$ (if there are multiple or no such positions, then the return value could be arbitrary).

This retrieval operation can be implemented with one attention head. In particular, we set the $V$-value to be $b_j$, $K$-value to be $\log^2(n) \cdot (a_j, \vec{1}-a_j)$ for position $j< i$;\footnote{Here, $\vec{1} \in \{0,1\}^D$ denotes the all-$1$ vector of length $D$, and $\vec{1} - a_j$ denotes element-wise subtraction.} the $Q$-value at position $i$ is taken to be $\log^2(n)\cdot (a, \vec{1}-a)$. The attention score (before softmax) satisfies
\[
\langle Qx_{i}^{(\ell)}, Kx_{j}^{(\ell)}\rangle = \left\{
\begin{matrix}
    \log^2 (n)D  & a_j = a\\
    \leq \log^2(n) D - \log^2(n) & a_j \neq a .
\end{matrix}
\right.
\]
Hence, if there is exactly one position $j \leq i$ that satisfies $a_j = a$, then the attention head would only attend to position $j$ and get the value $b_j$, assuming the precision $p=\Theta(\log(n))$. %\lijie{actually why? did you take the precision to be $\Theta(\log n)$ here?}

\paragraph{Proof of Corollary \ref{cor:depth-size}.} Consider an $(L+1)$-layer Transformer such that each layer has one attention head. For any $\ell \in [L+1]$, let the attention head at layer $\ell$ implements the retrieval task for the $(\ell-1)$-th composition, i.e., given $i_{\ell-2}$ and $w_{\ell-2}$ at the last token, find $i_{\ell-1} = z_{\ell-1}(w_{\ell-2}, i_{\ell-2})$. Concretely, the last token implements the query $a = (w_{\ell-2}, i_{\ell-2}) \in [N_{\ell-2}] $. For every previous token $j \in [n-1]$, if the $j$-th token corresponds to the $t$-th entry of $z_{\ell-1}$ ($t\in [N_{\ell-1}]$), then it sets $a_j = t$ and $b_j = z_{\ell-1}(t)$; otherwise, if the $j$-th token does not corresponds to any entry of $z_{\ell-1}$, it provides an arbitrary dummy pair of $a_j$ and $b_j$.
%generates a dummy pair of $(a_j, b_j)$ if does not , of $z_{\ell-1}$ stores the value $j$ and $z_{\ell-1}(j)$ ($j \in [N_{\ell-2}]$) and all other tokens use dummy pairs. 
One can inductively prove that the last token successfully retrieves the value of $i_{\ell-1}$ after layer $\ell$. 

\paragraph{The proof of Corollary \ref{cor:cot}.} 
The proof is similar to Corollary \ref{cor:depth-size}. We design $L+1$ attention heads, where the $\ell$-th attention head aims to retrieve $i_{\ell-1} = z_{\ell-1}(w_{\ell-2}, i_{\ell-2})$ given $w_{\ell-2}, i_{\ell-2}$. One can inductively prove that at the $\ell$-th CoT step, the Transformer could obtain $i_{\ell-1}$ and include it into the next token.

\paragraph{Proof of Corollary \ref{cor:encoder}.} 
Recall we want to construct an $O(\log(L))$-layer encoder that solves the sequential function composition task.
For the first layer, the value of $w$ can be shared over all positions $[n]$ using one attention head.
For layer $\ell \geq 2$, one can use one attention head to retrieve the value of $2^{\ell-2}$-hop composition value for every position, Hence, the value of $i_{L}$ can be obtained using at most $2\log_2(L)+1$ layers of attention. The proof is similar to Theorem 4.2 of \cite{sanford2024transformers}.\footnote{There is one minor difference: The $L$ functions of \cite{sanford2024transformers} are of the same size, whereas our $L$ functions are of different size (after fixing the value of $w$). Nevertheless, one can check that the proof still can go through.}
%\Binghui{Maybe we can cite Theorem 4.2 in \cite{sanford2024transformers}. There is a subtle different: \cite{sanford2024transformers} consider the composition of $f_1 \odot \cdots \odot f_k$ when $f_1, \ldots, f_k$ are explicitly given so $f_1, \ldots, f_k$ has the same size, in our case, even $w$ is shared, the $\ell$ functions are different.}
%\lijie{we can probably add a footnote say this is similar to that Theorem?}
%\Binghui{OK}

\section{Encoder lower bounds imply circuit lower bounds}\label{app:encoders-lowb-imply-ckt-lowb}

In this section, we will sketch a proof that a lower bound for \emph{encoder-only} transformers would imply breakthrough circuit lower bounds against constant-depth symmetric circuits. A constant-depth symmetric circuit is a constant depth circuit in which every gate has unbounded fan-in and computes a symmetric boolean function on its inputs. A function $f$ is symmetric if function value $f(x)$ only depends on the number of 1's in $x$. 

More formally, we will show that any depth-$L$ symmetric circuit with $s$ wires can be simulated by an encoder-only transformer of depth $6L$ and $Hdp = O(( s / n)^2 )$. Therefore, any lower bound of $Hdp = \Omega(n^{\epsilon})$ against encoders is also an $\Omega(n^{1 + \epsilon / 2})$-wire lower bound against symmetric circuits. 

Proving lower bounds against constant-depth symmetric circuits is notoriously hard.\footnote{A technical reason is that a very powerful technique for proving lower bounds against constant-depth circuits, the random restriction method, fails when applying to certain symmetric functions such as the parity function.} It was shown~\cite{RoychowdhuryOS94} that the inner product module $2$ function\footnote{Given input $x_1,\dotsc,x_{n},y_1,\dotsc,y_n \in \{0,1\}^{2n}$, compute $\sum_{i=1}^{n} x_i \cdot y_i \pmod{2}$.} requires a symmetric circuit of $\Omega(n)$ gates. However, no non-trivial lower bounds were known even for $O(n)$-wire depth-$3$ symmetric circuits (see, e.g.,~\cite{Tell2024note}).\footnote{For the case of depth-$2$ symmetric circuits, it was shown by~\cite{AlmanCW16} that $\mathsf{E}^\mathsf{NP}$ (exponential-time with access to an $\mathsf{NP}$ oracle, an extremely large complexity class) requires depth-$2$ circuits of at least $n^{2-\varepsilon}$ gates; no non-trivial results were known when we restrict the hard function to be in $\mathsf{NP}$.} Therefore, given our simulation result, if for any $\eps > 0$ we can prove an $\Omega(n^{\eps})$-size lower bound against depth-$18$ encoders, we would get lower bounds against $n^{1+\eps/2}$-wire depth-$3$ symmetric circuits, which would be a breakthrough in circuit complexity theory.

\paragraph*{On the choice of attention/MLP layers.} Here we allow the parameters for the attention and MLP layers to change arbitrarily for different positions (our auto-regressive communication model applies to this case as well). Otherwise, the encoder architecture would only have $\poly(Hdp)$ bits of non-uniformity; an encoder lower bound of $Hdp = \Omega(n^\epsilon)$ would then follow directly from the folklore time hierarchy theorem with advice $\mathsf{DTIME}(n^{k+1}) \not\subseteq \mathsf{DTIME}(n^k) / o(n)$ (see, e.g., Proposition 1 of~\cite{santhanam2013medium}).\footnote{We consider such a lower bound uninteresting because it does not say anything meaningful about the Transformer architecture other than it's not sufficiently non-uniform. An alternative approach of incorporating a super-linear amount of non-uniformity is to allow arbitrary positional encodings; we choose to work with position-dependent attention/MLP layers to make our presentation simpler.}

\paragraph*{Simulating a single symmetric gate.} On input $x_1, x_2, \dots, x_n \in \{0,1\}^n$, a symmetric gate with $w$ wires outputs $f(x_{i_1} + x_{i_2} + \cdots + x_{i_w})$ for some fixed function $f : \{0,1,2,\dots,w\} \to \{0,1\}$. 

First, suppose each input bit $x_j$, we can set the V-value to be $x_j$ and $K$-value to be $\log^2(n) \cdot (\vec{j}, \vec{1}-\vec{j})$ where $\vec{j} \in \{0,1\}^d$ is the binary representation of $j$. Similar to the retrieval head in \Cref{appendix:missing-proof}, this allows an attention head to only attend to position $j$ and get the value $x_j$. 

Second, we need to gather the sum $x_{i_1} + x_{i_2} + \cdots + x_{i_w}$.  We set the number of attention heads $H = 5s / n$ and $d = H$. Then we divide into two cases and solve each with two layers:
\begin{itemize}
    \item (\textbf{Handling small fan-in gates.}) If the fan-in of the symmetric gate $w \leq H$, at layer $f^1_{\attn}$, we gather the sum with $w$ attention heads of a single position $k \in [n]$ of the first level. Each attention head $j \in [w]$ attends to a single input position $i_j$ using the retrieval head and get value $x_{i_j}$. 
    
    Afterwards, let $y^1_k \in \mathbb{R}^{dH}$ denote the output of all $H$ attention heads.  We can choose the parameter of the MLP layer $f_{\mlp}^1: \mathbb{R}^{dH} \to \mathbb{R}^d$, so that it computes the linear sum over these attention heads and output it in an arbitrary entry $t$, that is $f_{\mlp}(y^1_k)_t = x_{i_1} + x_{i_2} + \cdots + x_{i_w}$. 

    In this case, we do not need a second layer. We can simply let $f^2_{\mlp} \circ f^2_\attn$ be identity by choosing $f^2_\attn$ to be retrieval heads and $f^2_{\mlp}$ be an identity MLP. 
    
    \item (\textbf{Handling large fan-in gates.}) If the fan-in of the symmetric gate $w > H$, without loss of generality we can assume that $w = H\cdot r$ for some $r \in \mathbb{Z}$ by padding and this blows up the number of wires by at most a factor of $2$. 

    At layer $f^1_{\mlp} \circ f^1_{\attn}$, we spread the sum to $w / H = r$ positions. For each $j \in [r]$, we use a position $k_j \in [n]$ to gather the partial sum $x_{i_{(j - 1) \cdot H + 1}} + \cdots + x_{i_{j \cdot H}}$. This can be done using the same construction as the first case. 

    Next, we will use a single attention head in $f^2_{\attn}$ to gather all the partial sums. Specifically, suppose $i$ is the index of the gate we are simulating and $\vec{i}$ is its binary representation. We will set the K-value to be $(\vec{i}, \vec{1} - \vec{i})$ and V-value be the partial sums. Suppose $p = \Theta(\log n)$, following the same analysis as the retrieval head, the query $Q = (\vec{i}, \vec{1} - \vec{i})$ will only attend to these partial sums and get their average. Since $r$ is fixed when the symmetric circuit is given, we can use $f^2_\mlp$ to multiply the average by $r$ and get $x_{i_1} + x_{i_2} + \cdots + x_{i_w}$. %\lijie{I can't follow this paragraph...} \hongxun{How about now?}\lijie{looks good now!}
\end{itemize}

Thirdly, we need to compute $f(x_{i_1} + x_{i_2} + \cdots + x_{i_w})$ from the sum. Here, because $f$ is a fixed function, we can implement a look-up table for $f$ similar to Section 3.2 of \cite{chiang2022overcoming}:
\begin{itemize}
    \item (\textbf{Handling small fan-in gates.}) If $w \leq H$, at layer $f_{\attn}^3$, we use a single attention head to retrieve the sum $x_{i_1} + x_{i_2} + \cdots + x_{i_w}$. Let $y^3_k \in \mathbb{R}^{dH}$ denote the output of all $H$ attention heads. In MLP layer $f_{\mlp}^3$, we copy the sum to $w$ different entries of $f_{\mlp}^3(y^3_k) \in \mathbb{R}^d$. 
    
    Then we let $f_{\attn}^4$ and $f_{\attn}^5$ be simple retrieval heads, so that $f_{\mlp}^5 \circ f_{\attn}^5 \circ f_{\mlp}^4 \circ f_{\attn}^4$ can simulate a two-layer MLP network. As shown in \cite{chiang2022overcoming}, such a two-layer network can compute the piecewise linear function that equals $\mathbbm{1}[x_{i_1} + x_{i_2} + \cdots + x_{i_w} = t]$ on the $t$-th entry ($t \in [w]$). Multiplying the corresponding weights by $f(t)$ gives $\mathbbm{1}[x_{i_1} + x_{i_2} + \cdots + x_{i_w} = t] \cdot f(t)$. 

    Finally, we use $f_{\attn}^6$ and $f_{\mlp}^6$ to sum $\mathbbm{1}[x_{i_1} + x_{i_2} + \cdots + x_{i_w} = t] \cdot f(t)$ over $t \in [w]$ and get $f(t)$. This is similar to how we gather the partial sums. 
    
    \item (\textbf{Handling large fan-in gates.}) Otherwise, when $w > H$, we again without loss of generality assume $w = H \cdot r = d \cdot r$ for $r \in \mathbb{Z}$. We again spread $\mathbbm{1}[x_{i_1} + x_{i_2} + \cdots + x_{i_w} = t] \cdot f(t)$ for different $t \in [w]$ to $r$ different positions. We set up $f^5_\mlp\circ f^5_\attn\circ f^4_\mlp\circ f^4_\attn\circ f^3_\mlp\circ f^{3}_\attn $ in the same way as in Case 1, so that the $j$-th ($j \in [r]$) position computes $\mathbbm{1}[x_{i_1} + x_{i_2} + \cdots + x_{i_w} = t] \cdot f(t)$ for all $t \in [(j-1)]\cdot d+1,j\cdot d]$. 
    
    Then we use $f^{6}_{\attn}$ and $f^6_{\mlp}$ to gather them. Let $R$ be the set of $r$ positions we use here. For each position $k \in R$, the output of $f^{5}_\mlp$ is a vector $x^5_k \in \mathbb{R}^d$ containing the value of $\mathbbm{1}[x_{i_1} + x_{i_2} + \cdots + x_{i_w} = t] \cdot f(t)$ for $d$ many $t$'s. For $f^6_\attn$, we gather $\frac{1}{r} \sum_{k\in R} x_k$ in the same way as how we gather the partial sums using $f^2_\attn$. Then we multiply it by $r$ and also sum over the $d$ entries with $f^6_\mlp$. This gives the value of $f(t)$ in a single entry in the output of $f^6_\mlp$.
    %\lijie{feels too sketchy...} \hongxun{Maybe now it's better?}\lijie{good good}
\end{itemize}

\paragraph*{Simulating a symmetric circuit.} We now apply the above simulation to each gate in one layer of our symmetric circuit. For those gates with fan-in $w$ less or equal to $H$, in the above construction, they only use $w$ attention heads and at most $w$ entries after the MLP (note $d = H = 5 s / n$). As each position has $H$ attention heads, we can pack as many as possible such gates at a single position $k \in [n]$. 

After packing, no two positions will simultaneously have fan-in less than $H / 2$ because otherwise we can simply pack them into a single position. In the end, all gates with fan-in less than $H$ uses at most $s / (H/2) + 1 \leq n / 2$ positions.

Then for those gates with fan-in $w$ bigger than $H$, our strategy of spreading it into $\lceil w / H\rceil < w / H + 1$ positions will use at most $2 \cdot s / H \leq n / 2$ positions as well. Thus the $n$ positions per layer encoder we have is sufficient for the simulation. Finally, stacking the simulation of each one of the $L$ layers together, we get a encoder with $6L$ layers and $Hdp = O((s / n)^2)$.
%\Binghui{change model size to $Hdp$}
\end{document}